\documentclass[twoside]{article}

\usepackage[accepted]{aistats2023}
\usepackage{graphicx}
\usepackage{amsmath}
\usepackage{amsfonts}
\usepackage{amsthm}
\usepackage{comment}
\usepackage{dsfont}
\newtheorem{definition}{Definition}
\newtheorem{assumption}{Assumption}
\newtheorem{lem}{Lemma}
\newtheorem{corr}{Corollary}
\newtheorem{theorem}{Theorem}
\usepackage[hidelinks]{hyperref}
\usepackage{caption}
\usepackage{subcaption}
\usepackage{hyperref}
\usepackage{url}
\usepackage{color}

\DeclareMathOperator*{\argmax}{argmax}
\DeclareMathOperator*{\argmin}{argmin}

\usepackage[round]{natbib}

\begin{document}

%

%

\twocolumn[

\runningtitle{Nash Equilibria and Pitfalls of Adversarial Training in Adversarial Robustness Games}
\aistatstitle{Nash Equilibria and Pitfalls of Adversarial Training\\ in Adversarial Robustness Games}

\aistatsauthor{ Maria-Florina Balcan  \And  Rattana Pukdee \And  Pradeep Ravikumar \And Hongyang Zhang}
\aistatsaddress{ Carnegie Mellon University   \And  Carnegie Mellon University \And Carnegie Mellon University \And University of Waterloo }
 ]

\begin{abstract}

   Adversarial training is a standard technique for training adversarially robust models.  In this paper, we study adversarial training as an alternating best-response strategy in a 2-player zero-sum game. We prove that even in a simple scenario of a linear classifier and a statistical model that abstracts robust vs. non-robust features, the alternating best response strategy of such game may not converge. On the other hand, a unique pure Nash equilibrium of the game exists and is provably robust. We support our theoretical results with experiments, showing the non-convergence of adversarial training and the robustness of Nash equilibrium.
  
\end{abstract}

\section{INTRODUCTION}

Deep neural networks have been widely applied to various tasks~\citep{lecun2015deep,goodfellow2016deep}. However, these models are vulnerable to human-imperceptible perturbations, which may lead to significant performance drop \citep{goodfellow2014explaining,szegedy2013intriguing}. A large body of works has improved the robustness of neural networks against such perturbations. For example, Adversarial Training \citep{madry2018towards} (AT) is a notable technique that trains a robust model by two alternative steps: 1) finding adversarial examples of training data against the current model; 2) updating the model to correctly classify the adversarial examples and returning to step 1). This procedure has a strong connection with an alternating best-response strategy in a 2-player zero-sum game. In particular, we consider a game between an adversary (row player) and a defender (column player). At each time $t$, a row player outputs a perturbation function that maps each data point to a perturbation, and a column player selects a model. Given a loss function, the utility of the row player is the expected loss of the model on the  perturbed data and the utility of the column player is the negative expected loss. Therefore, steps 1) and 2) correspond to both players' actions that maximize their utility against the latest action of the opponents.

In this work, we show that for the adversarial robustness game, even in a simple setting, the alternating best-response strategy may not converge. We consider a general symmetric independent distribution beyond the symmetric Gaussian distribution which was typically assumed in the prior works \citep{tsipras2018robustness,ilyas2019adversarial}. We call this game the Symmetric Linear Adversarial Robustness (SLAR) game. The challenge is that SLAR is not a convex-concave game, and so those known results on convex-concave zero-sum games do not apply in our setting. One of our key contributions is to analyze the dynamics of adversarial training in the SLAR game which sheds light on the behavior of adversarial training in general. On the other hand, we prove the existence of a pure Nash equilibrium and show that any Nash equilibrium provably leads to a robust classifier, i.e., a classifier that puts zero weight on the non-robust features. The Nash equilibrium is unique where any two Nash equilibria select the same classifier. Our finding motivates us to train a model that achieves a Nash equilibrium. 

For linear models, there is a closed-form solution of adversarial examples for each data point \citep{bose2020adversarial, tsipras2018robustness}. Different from the alternating best-response strategy, we also study the procedure of substituting the closed-form adversarial examples into the inner maximization problem and reducing the problem to a standard minimization objective. We refer to this procedure as Optimal Adversarial Training (OAT). \citep{tsipras2018robustness} has shown that OAT leads to a robust classifier under symmetric Gaussian distributions. We extend their results by showing that the same conclusion also holds for the SLAR game. We support our theoretical results with experiments, demonstrating that standard adversarial training does not converge while a Nash equilibrium is robust.

\section{RELATED WORK}
\subsection{Adversarial Robustness}\

Variants of adversarial training methods have been proposed to improve adversarial robustness of neural networks \citep{zhang2019theoretically, shafahi2019adversarial, rice2020overfitting, wong2019fast, xie2019feature, qin2019adversarial}. Recent works utilize extra unlabeled data \citep{carmon2019unlabeled,zhai2019adversarially,deng2021improving,rebuffi2021data} or synthetic data from a generative model \citep{gowal2021improving, sehwag2021robust} to improve the robust accuracy. Another line of works consider ensembling techniques \citep{tramer2018ensemble,sen2019empir,pang2019improving,zhang2022building}. A line of theoretical works analyzed adversarial robustness by linear models, from the trade-off between robustness and accuracy \citep{tsipras2018robustness,javanmard2020precise, raghunathan2020understanding}, to the generalization property \citep{schmidt2018adversarially}. Recent works further analyze a more complex class of models such as 2-layer neural networks \citep{allen2022feature,bubeck2021law,bartlett2021adversarial,bubeck2021universal}. \\
 
Specifically, prior works considered adversarial robustness as a 2-player zero-sum game \citep{pal2020game, meunier2021mixed,bose2020adversarial,bulo2016randomized,perdomo2019robust, pinot2020randomization}. For instance, \citet{pal2020game} proved that randomized smoothing \citep{cohen2019certified} and FGSM attack \citep{goodfellow2014explaining} form Nash equilibria. \citet{bose2020adversarial} introduced a framework to find adversarial examples that transfer to an unseen model in the same hypothesis class. \textcolor{black}{
\citet{pinot2020randomization} shows the non-existence of a Nash equilibrium in the adversarial robustness game when the classifier and the Adversary are both deterministic. However, our settings are different from prior papers. The key assumption in previous work \citep{pinot2020randomization} is that an adversary is regularized and would not attack if the adversarial example does not change the model prediction. We consider the case where the adversary attacks even though the adversarial example does not change the model prediction (as in the standard adversarial training).
}

While most works focused on the existence of Nash equilibrium and proposed algorithms that converge to the equilibrium, to the best of our knowledge, no prior works showed that \textit{a Nash equilibrium is robust}.

\subsection{Dynamics in Games}
The dynamics of a 2-player zero-sum game has been well-studied, especially when each player takes an alternating best-response strategy in the finite action space. A classical question is whether players' actions will converge to an equilibrium as the two players alternatively play a game \citep{nash1950equilibrium,nash1951non}. It is known that the alternating best-response strategy converges to a Nash equilibrium for many types of games, such as potential games \citep{monderer1996potential}, weakly acyclic games \citep{fabrikant2010structure}, aggregative games \citep{dindovs2006better}, super modular games \citep{milgrom1990rationalizability}, and random games \citep{heinrich2021best, amiet2021pure}. However, this general phenomenon may not apply to adversarial robustness games since this natural learning algorithm may not converge even in simple games \citep{balcan2012weighted}, as these results rely on specific properties of those games. In addition, there are also works on different strategies such as fictitious play \citep{brown1951iterative, robinson1951iterative, monderer1996fictitious, benaim1999mixed} and its extension to an infinite action space \citep{oechssler2001evolutionary, perkins2014stochastic} or continuous time space \citep{hopkins1999note,hofbauer2006best}. Furthermore, there is a connection between a 2-player zero-sum game with online learning where it is possible to show that an average payoff of a player with a sub-linear regret algorithm (such as  follow the regularized leader or follow the perturbed leader) converges to a Nash equilibrium \citep{cesa2006prediction,syrgkanis2015fast,suggala2020online}.
However, \citet{mertikopoulos2018cycles} studied the dynamics of such no-regret algorithms and showed that when both players play the follow-the-regularized-leader algorithms, the actions of each player do not converge to a Nash equilibrium with a cycling behavior in the game. 

\section{SETUP}\label{section: setup}
We consider a binary classification problem where we want to learn a linear function $f: \mathbb{R}^d \to \mathbb{R}$, such that our prediction is given by $\operatorname{sign}(f(x))$. Let $\mathcal{D}$ be the underlying distribution of $(x,y)$ and let $x = [x_1, \dots, x_d]$. We assume that the distribution of each feature $x_i$ has a symmetrical mean and is independent of the others given the label.
\begin{assumption}
(Symmetrical mean)
\label{assum: data distribution}
The mean of each feature $x_i$ is symmetrical over class $y = -1,1$. That is
    \begin{equation*}
        \mathbb{E}[x_i|y] = y\mu_i,
    \end{equation*}
    where $\mu_i$ is a constant.
\end{assumption}
\begin{assumption}
(Independent features given the label)
\label{assum: independent features}
Each feature $x_i$ is independent of each other given the label.
\end{assumption}
We study this problem on a soft-SVM objective

\begin{equation}
    \label{eq: svm objective}
    \min_w \mathcal{L}(w),
\end{equation}
where
\begin{equation*}
     \mathcal{L}(w) = \mathbb{E}_{(x,y) \sim \mathcal{D}}[\max(0, 1 - yw^{\top}x)] + \frac{\lambda}{2}||w||^2_2,
\end{equation*}
and $\lambda$ is the regularization parameter. Assume that at the test time, we have an adversarial perturbation function $\delta: \mathcal{D} \to \mathcal{B}(\varepsilon)$, $\mathcal{B}(\varepsilon) = \{a: ||a||_{\infty} \leq \varepsilon \}$, that adds a perturbation $\delta(x,y)$ to a feature of each point $(x,y)$. Our goal is to learn a function $f$ that makes correct predictions on the perturbed data points. We denote $\varepsilon$ as the perturbation budget. 

 For a given perturbation budget $\varepsilon$, we divide all features $x_i$'s into robust features and non-robust features.
\begin{definition}[Non-robust feature]
\label{def: non-robust feature}
A feature $x_i$ is non-robust when the perturbation budget is larger than or equal to the mean of that feature
 \begin{equation*}
     |\mu_i| \leq \varepsilon.
 \end{equation*}
 Otherwise, $x_i$ is a robust feature.
 \end{definition}
 We discuss non-robust features in more details in Appendix \ref{appendix: non-robust feature}.

\subsection{Symmetric Linear Adversarial Robustness Game}\label{section: SLAR game}
\begin{figure*}[ht]
    \centering
    \includegraphics[width =1\textwidth]{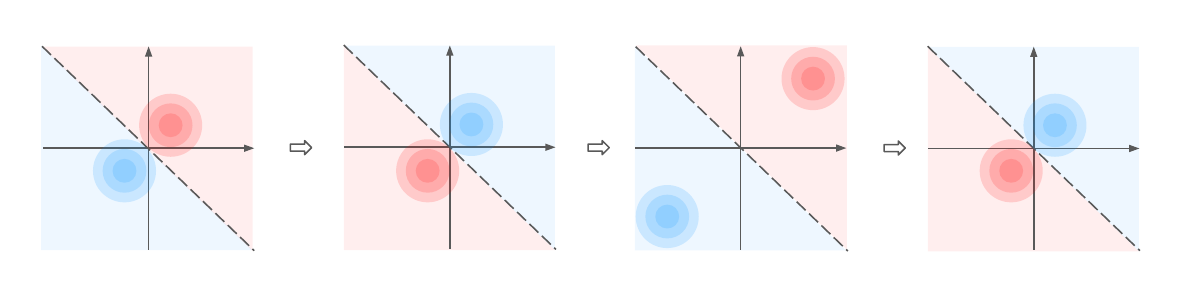}
    
    \caption{The space of two non-robust features, where the red and blue circles are for class $y = -1, 1$, respectively. Dashed lines represent a decision boundary of each model and background colors represent their prediction. The figure on the left is the original distribution. The adversary can always shift the mean of non-robust features across the decision boundary. A model trained with adversarial examples has to flip the decision boundary at every iteration. For example, the adversary shifts the red and blue circle across the original decision boundary leading to a decision boundary flip for a model trained on the perturbed data (second figure from the left). }

    \label{fig: flip flop pic}
\end{figure*}
We formulate the problem of learning a robust function $f$ as a 2-player zero-sum game between an adversary (row player) and a defender (column player). The game is played repeatedly where at each time $t$, the row player outputs a perturbation function $\delta^{(t)} : \mathcal{D} \to \mathcal{B}(\varepsilon)$ that maps each data point in $D$ to a perturbation while the column player outputs a linear function $f^{(t)} = (w^{(t)})^\top x$. The utility of the row player is given by
\begin{equation*}
    U_{\operatorname{row}}(\delta^{(t)},w^{(t)}) := 
    \mathbb{E}_{(x,y) \sim \mathcal{D}}[l(\delta^{(t)},w^{(t)},x,y)] + \frac{\lambda}{2}||w^{(t)}||^2_2,
\end{equation*}
where
\begin{equation*}
    l(\delta, w,x,y) = \max(0, 1 - yw^\top (x + \delta(x,y))).
\end{equation*}
The goal of the row player is to find a perturbation function that maximizes the expected loss of the perturbed data given a model from the column player. The utility of the column player is the negative expected loss:
\begin{equation*}
    U_{\operatorname{col}}(\delta^{(t)},w^{(t)}) = -U_{\operatorname{row}}(\delta^{(t)},w^{(t)}),
\end{equation*}
where the column player wants to output a model that minimizes the expected loss given the perturbed data.

\subsection{Adversarial Training as An Alternating Best-Response Strategy}
Recall that in AT \citep{madry2018towards}, we first find the ``adversarial examples'', i.e., the perturbed data points that maximize the loss. We then optimize our model according to the given adversarial examples. In the game-theoretic setting, AT is an alternating best-response strategy:
\begin{enumerate}
    \item The row player submits a perturbation function that maximizes the utility from the last iteration:
    \begin{equation*}
    \delta^{(t)} = \argmax_{\delta: \mathcal{D} \to \mathcal{B}(\varepsilon)} U_{\operatorname{row}}(\delta,w^{(t-1)}).
    \end{equation*}
    \item The column player chooses a model that maximizes the utility given the perturbation  $\delta^{(t)}$:
    \begin{equation*}
    w^{(t)} = \argmax_{w \in \mathbb{R}^d} U_{\operatorname{col}}(\delta^{(t)}, w).
    \end{equation*}
\end{enumerate}
In practice, we achieve an approximation of the $w^{(t)}$ via stochastic gradient descent and an approximation of each instance $\delta^{(t)}(x,y)$ by projected gradient descent \citep{madry2018towards}.

\section{NON-CONVERGENCE OF ADVERSARIAL TRAINING}

In this section, we start with the dynamics of AT on a SLAR game. We then provide an example of a class of data distributions on which AT does not converge. A key property of such distributions is that it has a large fraction of non-robust features.

It is known that we have a closed form solution for the worst-case adversarial perturbation w.r.t. a linear model \citep{bose2020adversarial, tsipras2018robustness}.
\begin{lem}
\label{possible perturb}
For a fixed $w$, for any $(x,y) \sim D$, the perturbation $\delta(x,y) = -y\varepsilon \operatorname{sign}(w)$ maximizes the inner optimization objective
\begin{equation*}
    \max_{\delta \in \mathcal{B}(\varepsilon)}\max(0, 1 - yw^{\top}(x + \delta)),
\end{equation*}
where
\begin{equation*}
    \operatorname{sign}(x)=\left\{\begin{array}{cl}
1, & \text{if } x>0; \\
0, &  \text{if } x=0; \\
-1, &  \text{if } x<0.
\end{array}\right.
\end{equation*}
When $x$ is a vector, $\operatorname{sign}(x)$ is applied to each dimension. We denote this as the worst-case perturbation.
\end{lem}
We note that the worst-case perturbation does not depend on the feature $x$, which means any point in the same class has the same perturbation. Intuitively, the worst-case perturbation shifts the distribution of each class toward the decision boundary. Since there is no other incentive for the adversary to choose another perturbation, we assume that the AT always picks the worse-case perturbation $\delta(x,y) = -y\varepsilon \operatorname{sign}(w)$. 
This implies that at time $t$, the row player submits the perturbation function $\delta^{(t)}$ such that
\begin{equation*}
    \delta^{(t)}(x,y) = -y\varepsilon\operatorname{sign}(w^{(t-1)}).
\end{equation*}
However, later in this work, we do not restrict our action space to only the worst-case perturbations when analyzing a Nash equilibrium. Now, we can derive the dynamics of AT. We prove that for non-robust features $x_i$'s, if a column player puts a positive (negative) weight of the model on $x_i$ at time $t$, then the model at time $t+1$ will put a non-positive (non-negative) weight on $x_i$.
\begin{theorem}[Dynamics of AT]
\label{thm: AT cycle}
 Consider applying AT to learn a linear model $f(x) = w^{\top}x$. Let $w^{(t)} = [w^{(t)}_1,w^{(t)}_2, \dots, w^{(t)}_{d}]$ be the parameter of the linear function at time $t$. For a non-robust feature $x_i$,
\begin{enumerate}
    \item If $w^{(t)}_i > 0$, we have $w^{(t+1)}_i \leq 0$;
    \item If $w^{(t)}_i <0 $, we have $w^{(t+1)}_i \geq 0$,
\end{enumerate}
for all time $t > 0$. 
\end{theorem}
\begin{proof}
 The key intuition is that mean of non-robust features is smaller than the perturbation budget, and the adversary can always shift the mean of these features across the decision boundary. Therefore, if we want to train a model to fit the adversarial examples, we have to flip the decision boundary at every iteration (Figure \ref{fig: flip flop pic}). Formally, consider a non-robust feature $x_i$ with $w^{(t)}_i > 0$, the perturbation at time $t+1$ of feature $x_i$ is given by
 \begin{equation*}
    \delta^{(t+1)}_i(x,y) = -y\varepsilon\operatorname{sign}(w^{(t)}_i) = -y\varepsilon.
\end{equation*}
The mean of the feature $x_i$ of the adversarial examples at time $t+1$ of class $y = 1$ is given by
\begin{equation*}
    \mu_i^{(t+1)} = \mathbb{E}[x_i + \delta_i^{(t+1)}(x,y)|y = 1] = \mu_i - \varepsilon \leq |\mu_i| - \varepsilon < 0.
\end{equation*}
The final inequality holds because $x_i$ is a non-robust feature. We note that for a linear classifier under SVM-objective, when the mean $\mu_i^{(t+1)} < 0$ we must have $w^{(t+1)}_i \leq 0$ (see Lemma \ref{lemma: sign w}).
\end{proof}
Theorem \ref{thm: AT cycle} implies that the difference between the consecutive model weight is at least the magnitude of weight on non-robust features.
\begin{corr}
 Consider applying AT to learn a linear model $f(x) = w^{\top}x$. Let $w^{(t)} = [w^{(t)}_1,w^{(t)}_2, \dots, w^{(t)}_{d}]$ be the parameters of the linear function at time $t$. We have
 \begin{equation*}
     ||w^{(t+1)} - w^{(t)}||_2^2 \geq \sum_{|\mu_i| < \varepsilon}(w_i^{(t)})^2.
 \end{equation*}
\end{corr}
 If a large fraction of the model weight is on the non-robust features at each time $t$, then the model will not converge. We provide an example of data distributions where, when we train a model with AT, our model will always rely on the non-robust features at time $t$. We consider the following distribution:

\begin{definition}[Data distribution with a large fraction of non-robust features]
\label{def: distribution for AT}
Let the data distribution be as follows
\begin{enumerate}
    \item $y {\sim} \operatorname{unif}\{-1,+1\}$,
    \item $
   x_{1}=\left\{\begin{array}{ll}
+y, & \text { w.p. } p; \\
-y, & \text { w.p. } 1-p,
\end{array}\right.
$
    \item $x_j|y$ is a distribution with mean $y\mu_j$ and variance $\sigma_j^2$, where $\varepsilon >\mu_j > 0$ , for $j = 2,3,\dots, d+1$.
\end{enumerate}
\end{definition}
Given a label $y$, the first feature $x_1$ takes 2 possible values, $y$ with probability $p$ and $-y$ with probability $1-p$. We note that $x_1$ is the true label with probability $p$ and is robust to adversarial perturbations. On the other hand, for $j\geq 2$, feature $x_j$ is more flexible where it can follow any distribution with a condition that the feature must be weakly correlated with the true label in expectation. Each feature might be less informative compared to $x_1$ but combining many of them can lead to a highly accurate model. We note that $x_j$ is non-robust since its mean is smaller than the perturbation budget.

The data distribution is a specification of our setup in Section \ref{section: setup} where we have a significantly large number of non-robust features compared to the robust features. This distribution is inspired by the one studied in \citep{tsipras2018robustness}, where they showed that standard training on this type of distribution (when features $j=2,\dots,d+1$ are Gaussian distributions) leads to a model that relies on non-robust features. We generalize their result to a scenario when $x_j$ can follow any distribution (see Appendix \ref{appendix: standard training relies on non-robust}). We note that the lack of assumption on each distribution is a key technical challenge for our analysis.

\begin{theorem}[Adversarial training uses non-robust feature (simplified version)]
\label{thm: adversarial training uses non-robust feature (sim)}
Let the data distribution follows
the distribution as in Definition \ref{def: distribution for AT}. Consider applying AT to learn a linear model $f(x) = w^{\top}x$. Assume that $\varepsilon > 2\mu_j$ for $j=2,\dots, d+1$. Let $w^{(t)} = [w^{(t)}_1,w^{(t)}_2, \dots, w^{(t)}_{d+1}]$ be the parameter of the linear function at time $t$. If 
\begin{equation}
\label{eq: adversarial training uses non-robust feature (sim)}
    p < 1 - \left(\frac{1}{2}(\frac{\sigma_{\max}}{||\mu' ||_2} + \frac{\lambda}{2||\mu'||_2^2}) + \frac{1}{2}\sqrt{\frac{2}{\lambda}}\sigma_{\max}\right),
\end{equation}
where
\begin{equation*}
    \sigma_i \leq \sigma_{\max}, \quad \mu' = [0, \mu_2, \dots, \mu_{d+1}],
\end{equation*}
then 
\begin{equation*}
    \sum_{j=2}^{d+1} (w^{(t)}_j)^2 \geq \frac{||w^{(t)}||_2^2(1-\varepsilon)^2}{(1-\varepsilon)^2 + \sum_{j=2}^{d+1} (\mu_j + \varepsilon)^2}.
\end{equation*}
(For simplicity, we also assume that $\sigma_{\max} \geq 1$. For a tighter bound, see Appendix \ref{appendix: adversarial training relies on non-robust}).
\end{theorem}
Since, features $j = 2, \dots, d+1$ are not robust, Theorem \ref{thm: adversarial training uses non-robust feature (sim)} implies that a model at time $t$ will put at least 
\begin{equation*}
    \frac{(1-\varepsilon)^2}{(1-\varepsilon)^2 + \sum_{j=2}^{d+1} (\mu_j + \varepsilon)^2}
\end{equation*}
fraction of weight on non-robust features. We note that the term $||\mu'||_2$ at the denominator of condition \eqref{eq: adversarial training uses non-robust feature (sim)} grows as $\mathcal{O}(d)$. Therefore, if the number of non-robust features $d$ and the regularization parameter $\lambda$ is large enough then the condition \eqref{eq: adversarial training uses non-robust feature (sim)} holds. We discuss the full version of this theorem in Appendix \ref{appendix: adversarial training relies on non-robust}. Next, we prove that the magnitude $||w^{(t)}||_2$ is bounded below by a constant (see Lemma \ref{lem: lower bound on the magnitude}) which implies that $||w^{(t)}||_2$ does not converge to zero. Therefore, we can conclude that a model trained with AT puts a non-trivial amount of weights on non-robust features at each iteration. This implies that AT does not converge.

\begin{theorem}[AT does not converge (simplified)]
\label{theorem: AT does not converge (simplified)}
Let the data follow the distribution as in Definition \ref{def: distribution for AT}. Assume that the variance $\sigma_j^2$ is bounded above and $\varepsilon > 2\mu_j$ for $j=2,\dots, d+1$. Consider applying AT to learn a linear model $f(x) = w^{\top}x$ on the SVM objective Let $w^{(t)}$ be the parameter of the linear function at time $t$. If the number of non-robust feature $d$ and the regularization parameter $\lambda$ is large enough then $w^{(t)}$ does not converge as $t\to \infty$. 
\end{theorem}

\section{PROPERTIES OF THE NASH EQUILIBRIUM}
 Recall the definition of a Nash equilibrium:
\begin{definition}[Nash Equilibrium]
 For a SLAR game, a pair of actions $(\delta^*, w^*)$ is called a pure strategy Nash equilibrium if the following hold
\begin{equation*}
    \sup_{\delta} U_{\operatorname{row}}(\delta,w^*) \leq U_{\operatorname{row}}(\delta^*,w^*) \leq \inf_{w} U_{\operatorname{row}}(\delta^*,w).
\end{equation*}
\end{definition}
From this definition, we note that for any fixed $w^*$, the inequality
\begin{equation*}
    \sup_{\delta} U_{\operatorname{row}}(\delta,w^*) \leq U_{\operatorname{row}}(\delta^*,w^*)
\end{equation*}
holds if and only if $\delta^*$ is optimal. We know that the worst-case perturbation $\delta^*(x,y) = -y\varepsilon\operatorname{sign}(w^*)$ satisfies this condition. However, the optimal perturbations might not be unique because of the $\max(0, \cdot)$ operator in the hinge loss. For instance, for a point that is far away from the decision boundary such that the worst-case perturbation leads to a zero loss: 
\begin{equation*}
     1 - yw^{\top}(x + \delta^*(x,y)) \leq 0,
\end{equation*}
any perturbation $\delta(x,y)$ will lead to a zero loss:
\begin{equation*}
    1 - yw^{\top}(x + \delta(x,y)) \leq 0.
\end{equation*}
Therefore, any perturbation $\delta$ leads to the same utility (Figure \ref{fig: perturbation are not unique}). On the other hand, if the worst-case perturbation leads to a positive loss, we can show that the optimal perturbation must be the worst-case perturbation.
\begin{figure}[ht]
    \centering
    \includegraphics[width =0.6\columnwidth]{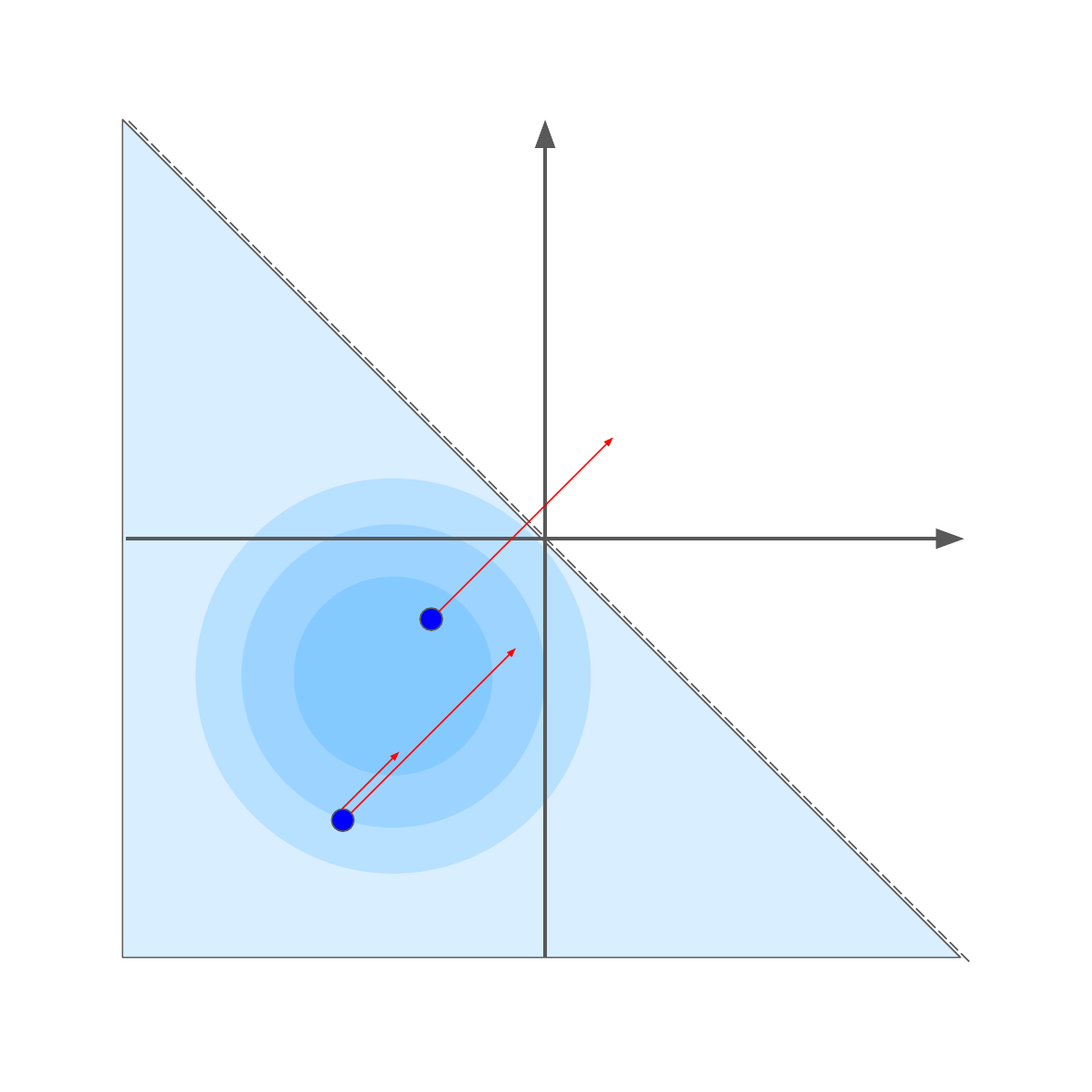}
    \caption{For points that are far enough from the decision boundary, any perturbation is optimal.}
    \label{fig: perturbation are not unique}
\end{figure}

\begin{lem}
\label{lem: optimal perturbation short}
For a fixed $w$ and any $(x,y) \sim D$, if
\begin{equation*}
     1 - yw^{\top}(x - y\varepsilon\operatorname{sign}(w)) > 0,
\end{equation*}
then the optimal perturbation in is uniquely given by $\delta^*(x,y) = -y\varepsilon \operatorname{sign}(w)$. Otherwise, any perturbation $\mathcal{B}(\varepsilon)$ is optimal.
\end{lem}
On the other hand, for any fixed $\delta^*$, the inequality
\begin{equation*}
    U_{\operatorname{row}}(\delta^*,w^*) \leq \inf_{w} U_{\operatorname{row}}(\delta^*,w)
\end{equation*}
holds when $w^*$ is an optimal solution of a standard SVM objective on the perturbed data distribution $(x + \delta^*(x,y), y) \sim D + \delta^*$. We know that for a fixed $\delta^*$, we have a unique $w^*$ (Lemma \ref{lem: unique solution of SVM}). Now, we show that a Nash equilibrium exists for the SLAR game.
\begin{theorem}[Existence of Nash equilibrium]
\label{thm: NE exists}
A Nash equilibrium exists for the SLAR game.
\end{theorem}
\begin{proof}
We will prove this by construction. Without loss of generality, let feature $x_i$ be a robust feature with $\mu_i > 0$ for $i = 1,\dots, k$ and let feature $x_j$ be a non-robust feature for $j = k+1, \dots, d$. Consider a perturbation function $\delta^*$ such that
\begin{equation*}
    \delta^*(x,y) = [\underbrace{-y\varepsilon,\dots, -y\varepsilon}_{k}, \underbrace{-y\mu_{k+1}, \dots, -y\mu_d}_{d-k}].
\end{equation*}
Intuitively, $\delta^*$ shifts the distribution of non-robust features by the same distance of their mean so that non-robust features have zero mean, in the perturbed data. Let $w^*$ be an optimal solution of a standard SVM objective on the perturbed data distribution $D + \delta^*$, which is known to be unique (Lemma \ref{lem: unique solution of SVM}). We will show that a pair $(\delta^*, w^*)$ is a Nash equilibrium. By the definition of $w^*$, it is sufficient to show that 
\begin{equation*}
    \sup_{\delta} U_{\operatorname{row}}(\delta,w^*) \leq U_{\operatorname{row}}(\delta^*,w^*).
\end{equation*}
First, we will show that $w^*_i \geq 0$ for $i = 1,\dots,k$ and $w^*_j = 0$ for $j = k+1,\dots, d$. We consider the mean of each feature on the perturbed data,
\begin{enumerate}
    \item For a robust feature $x_i$, we have 
    \begin{equation*}
        \mathbb{E}[x_i + \delta^*_i(x,y)|y=1] = \mu_i - \varepsilon > 0.
    \end{equation*}
    \item For a non-robust feature $x_j$, we have
    \begin{equation*}
         \mathbb{E}[x_j + \delta^*_j(x,y)|y=1] = \mu_j - \mu_j = 0.
    \end{equation*}
\end{enumerate}
From Lemma \ref{lemma: sign w} and Corollary \ref{lemma: w_i = 0 when mu_i = 0} we can conclude that $w^*_i \geq 0$ for $i = 1,\dots,k$ and $w^*_j = 0$ for $j = k+1,\dots, d$. Next, we will show that $\delta^*$ is optimal. Recall that for a fixed model $w^*$, the worst-case perturbation is given by
\begin{align*}
    \delta(x,y) &= -y\varepsilon\operatorname{sign}(w^*)\\
    &=[-y\varepsilon\operatorname{sign}(w^*_1),\dots, -y\varepsilon\operatorname{sign}(w^*_k), \underbrace{0, \dots, 0}_{d-k}].
\end{align*}
Although $\delta(x,y) \neq \delta^*(x,y)$, we note that if $w^*_i = 0$, any perturbation of a feature $x_i$ would still lead to the same loss
\begin{equation*}
    w^*_i(x_i + \delta_i(x,y)) = w^*_i(x_i + \delta^*_i(x,y)) = 0.
\end{equation*}
This implies that
\begin{equation*}
    U_{\operatorname{row}}(\delta^*,w^*) = U_{\operatorname{row}}(\delta,w^*) = \sup_{\delta} U_{\operatorname{row}}(\delta,w^*).
\end{equation*}
Therefore, $\delta^*$ is also an optimal perturbation function. We can conclude that $(\delta^*, w^*)$ is a Nash equilibrium.
\end{proof}
We can show further that for any Nash equilibrium, the weight on non-robust features must be zero. 
\begin{theorem}[Nash equilibrium is robust]
\label{thm: NE is robust}
Let $(\delta^*, w^*)$  be a Nash equilibrium of the SLAR game. For a non-robust feature $x_i$ we must have $w_i^* = 0$.
\end{theorem}
\begin{proof} (Sketch)
Let $(\delta^*,w^*)$ be a Nash equilibrium. Let $x_i$ be a non-robust feature. We will show that $w_i^* = 0$ by contradiction. Without loss of generality, let $w_i^* > 0$. Let the risk term in the SVM objective when $w_i = w$, $w_j = w_j^*$ for $j \neq i$ and $\delta = \delta^*$ be
\begin{align*}
    \mathcal{L}_i(w|w^*, \delta^*) &:= \mathbb{E}[l_i(x, y, w|w^*, \delta^*)].
\end{align*}
when 
\begin{align*}
    l_i(x, y, w|w^*, \delta^*) &= \max(0, 1 - y\sum_{j \neq i}w_j^*(x_j + \delta_j^*(x,y))) \\&\quad- yw(x_i + \delta_i^*(x,y)).
\end{align*}
We will show that when we set $w_i^* = 0$, the risk term does not increase, that is,
\begin{equation*}
    \mathcal{L}_i(w_i^*|w^*, \delta^*) \geq \mathcal{L}_i(0|w^*, \delta^*).
\end{equation*}
We use $l_i(x,y,w)$ to refer to $ l_i(x,y, w|w^*, \delta^*)$ for the rest of this proof. Considering each point $(x,y)$, we want to bound the difference
\begin{equation*}
    l_i(x,y,w_i^*) - l_i(x,y,0).
\end{equation*}
The key idea is to utilize the optimality of $\delta^*(x,y)$. From Lemma \ref{lem: optimal perturbation short}, we know that when the worst-case perturbation leads to a positive loss, the perturbation $\delta^*(x,y)$ must be the worst-case perturbation. If
\begin{equation*}
    1 - y\sum_{j} w_j^* x_j + \varepsilon \sum_{j} |w_j^*| > 0,
\end{equation*}
we must have
\begin{equation*}
    \delta_j^*(x,y) = -y\varepsilon \operatorname{sign}(w_j),
\end{equation*}
for all $j$. For example, assume this is the case we have 2 sub-cases\\

\noindent\textbf{Case 1.1: }
\begin{equation*}
    1 - y\sum_{j\neq i} w_j^* x_j + \varepsilon \sum_{j \neq i} |w_j^*| \geq 0.
\end{equation*}\\
In this case, we have
\begin{align*}
    &l_i(x,y,w_i^*) - l_i(x,y,0) \\
    &= \max(0, 1 - y\sum_{j }w_j^*(x_j + \delta_j^*(x,y)) ) \\&\quad- \max(0, 1 - y\sum_{j \neq i}w_j^*(x_j + \delta_j^*(x,y)) ) \\
    &= \max(0, 1 - y\sum_{j} w_j^* x_j + \varepsilon \sum_{j} |w_j^*| ) \\&\quad- \max(0, 1 - y\sum_{j \neq i} w_j^* x_j + \varepsilon \sum_{j \neq i} |w_j^*| ) \\
    &= (1 - y\sum_{j} w_j^* x_j + \varepsilon \sum_{j} |w_j^*|) \\&\quad- (1 - y\sum_{j \neq i} w_j^* x_j + \varepsilon \sum_{j \neq i} |w_j^*|) \\
    &= -yw_i^*x_i + \varepsilon |w_i^*|.
\end{align*} \\
We observe that as $x_i$ is a non-robust feature, we have
\begin{equation}
\label{eq: lower bound case1.1}
    \mathbb{E}[-yw_i^*x_i + \varepsilon |w_i^*|] \geq |w_i^*|(\varepsilon - \mu_i) > 0.
\end{equation}

\noindent\textbf{Case 1.2: }
\begin{equation*}
    1 - y\sum_{j\neq i} w_j^* x_j + \varepsilon \sum_{j \neq i} |w_j^*| < 0.
\end{equation*}\\
In this case, we have
\begin{align*}
    &l_i(x, y,w_i^*) - l_i(x,y, 0) \\&= \max(0, 1 - y\sum_{j }w_j^*(x_j + \delta_j^*(x,y)) ) \\&\quad- \max(0, 1 - y\sum_{j \neq i}w_j^*(x_j + \delta_j^*(x,y)) ) \\
    &= \max(0, 1 - y\sum_{j} w_j^* x_j + \varepsilon \sum_{j} |w_j^*| ) \\&\quad- \max(0, 1 - y\sum_{j \neq i} w_j^* x_j + \varepsilon \sum_{j \neq i} |w_j^*| ) \\
    &= (1 - y\sum_{j} w_j^* x_j + \varepsilon \sum_{j} |w_j^*|) - 0 \\
    &\geq 0.
\end{align*}\\
From Equation \eqref{eq: lower bound case1.1}, the lower bound in this Case 1.1 is positive in expectation. However, we note that the condition depends on the rest of the feature $x_j$ when $j\neq i$, so we can't just take the expectation.
In addition, there is also a case when the optimal perturbation is not unique, that is when
\begin{equation*}
    1 - y\sum_{j} w_j^* x_j + \varepsilon \sum_{j} |w_j^*| \leq 0.
\end{equation*}
We handle these challenges in the full proof in Appendix \ref{appendix: Nash equilibrium is robust}. Once we show that the risk term in the utility when $w^*_i \neq 0$ is no better than when $w^*_i = 0$, we note that the regularization term when $w^*_i = 0$ is higher, 
\begin{equation*}
    \frac{\lambda}{2}\sum_{j}(w_j^*)^2 > \frac{\lambda}{2}\sum_{j \neq i} (w_j^*)^2.
\end{equation*}
Therefore, we can reduce the SVM objective by setting $w_i^* = 0$. This contradicts the optimality of $w^*$. By contradiction, we can conclude that if a feature $i$ is not robust, then $w_i^* = 0$.
\end{proof}
Furthermore, we can show that any two Nash equilibria output the same model. 
\begin{theorem}[Uniqueness of Nash equilibrium]
\label{thm: NE unique}
Let $(\delta_u, u), (\delta_v, v)$ be Nash equilibrium of the SLAR game then we have $u = v$.
\end{theorem}
\begin{proof}
Let $(\delta_u, u), (\delta_v, v)$ be Nash equilibrium of the SLAR game. We know that
\begin{align*}
    U_{\operatorname{row}}(\delta_v, u) \leq U_{\operatorname{row}}(\delta_u,u) \leq U_{\operatorname{row}}(\delta_u, v),
\end{align*}
and
\begin{align*}
    U_{\operatorname{row}}(\delta_u, v) \leq U_{\operatorname{row}}(\delta_v,v) \leq U_{\operatorname{row}}(\delta_v, u).
\end{align*}
From these inequalities, we must have
\begin{align*}
    U_{\operatorname{row}}(\delta_v, u) = U_{\operatorname{row}}(\delta_u,u) =
    U_{\operatorname{row}}(\delta_u, v) = U_{\operatorname{row}}(\delta_v,v).
\end{align*}
From Lemma \ref{lem: unique solution of SVM}, we know that for a given perturbation function, we have a unique solution of the SVM objective on the perturbed data. Therefore,
\begin{align*}
    U_{\operatorname{row}}(\delta_u,u) =
    U_{\operatorname{row}}(\delta_u, v) 
\end{align*}
implies that we must have $u = v$.
\end{proof}
Since we have a construction for a Nash equilibrium in Theorem \ref{thm: NE exists}, Theorem \ref{thm: NE unique} implies that any Nash equilibrium will have the same model parameter as in the construction. This also directly implies that any Nash equilibrium is a robust classifier.

\subsection{Optimal Adversarial Training}
We note that in the SLAR game, we have a closed-form solution of worst-case perturbations in terms of model parameters,
\begin{align*}
    \delta^*(x,y) = -y\varepsilon\operatorname{sign}(w).
\end{align*}
We can substitute this to the minimax objective 
\begin{align*}
    \min_w \max_{\delta} U_{\operatorname{row}}(\delta, w),
\end{align*}
and directly solve for a Nash equilibrium. The objective is then reduced to a minimization objective
\begin{align*}
    \min_w \mathbb{E}[\max(0, 1 - yw^\top (x -y\varepsilon\operatorname{sign}(w)))] + \frac{\lambda}{2}||w||^2_2.
\end{align*}
We denote this as Optimal Adversarial Training (OAT). We note that \citep{tsipras2018robustness} analyze OAT when the data distribution is Gaussian distributions and show that directly solving this objective lead to a robust model. We further show that OAT also leads to a robust model for any SLAR game.
\begin{theorem}[Optimal adversarial training leads to a robust model]
In the SLAR game, let $w^* = [w_1^* , \dots, w_d^*]$ be a solution of OAT then for a non-robust feature $x_i$, we have $w_i^* = 0$.
\label{thm: oracle adversarial training is robust}
\end{theorem}
We defer the proof to Appendix \ref{appendix: OAT is robust}.




\section{EXPERIMENTS}
We illustrate that the theoretical phenomenon also occurs in practice. We provide an experiment comparing the convergence and robustness of AT and OAT on a synthetic dataset and MNIST dataset.
\subsection{Synthetic dataset}
Though our theoretical finding works for much broader data distributions, the construction of our experimental setup is as follows
\begin{enumerate}
    \item $y {\sim} \operatorname{unif}\{-1,+1\}$,
    \item $
   x_{1}=\left\{\begin{array}{ll}
+y, & \text { w.p. } p; \\
-y, & \text { w.p. } 1-p,
\end{array}\right.
$
    \item $x_j|y \sim \mathcal{N}(y\mu, \sigma^2).$
\end{enumerate}
We choose parameters $d = 2,000$, $p = 0.7$, $\mu = 0.01$, $\sigma = 0.01$ and set the perturbation budget $\varepsilon = 0.02$. This is an example of a distribution from Definition \ref{def: distribution for AT}. The size of the training and testing data is 10,000 and 1,000, respectively. We use SGD with an Adam optimizer \citep{kingma2014adam}  to train models. For more details, we refer to Appendix \ref{appendix: synthetic experiment}

\begin{figure}[ht]
    \centering
    \includegraphics[width =\columnwidth]{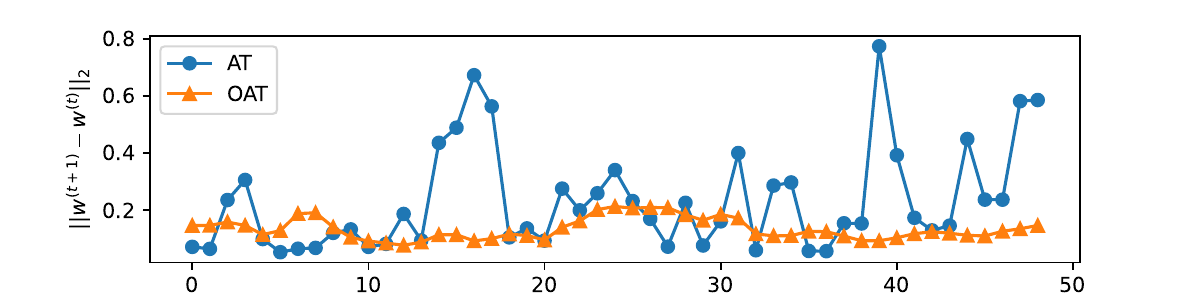}
    \caption{ $||w^{(t+1)} - w^{(t)}||_2$ of a linear model trained with AT and OAT.}
    \label{fig: SVM diff}
\end{figure}
\textbf{Non-convergence of Adversarial Training. } First, we calculate the difference between weight $||w^{(t+1)} - w^{(t)}||_2$ for each timestep $t$. We can see that for a model trained with AT, $||w^{(t+1)} - w^{(t)}||_2$ is fluctuating while the value from a model trained with OAT is more stable ( Figure \ref{fig: SVM diff}). 

\textbf{Robustness. } We investigate the robustness of each strategy. Since our model is linear, it is possible to calculate the distance between a point and the model's decision boundary. If the distance exceeds the perturbation budget $\varepsilon$ then we say that the point is certifiably robust.
We found that while the model trained with AT achieves a perfect standard accuracy, the model always achieves zero robust accuracy. 

\begin{figure}[ht]
    \centering
    \includegraphics[width =\columnwidth]{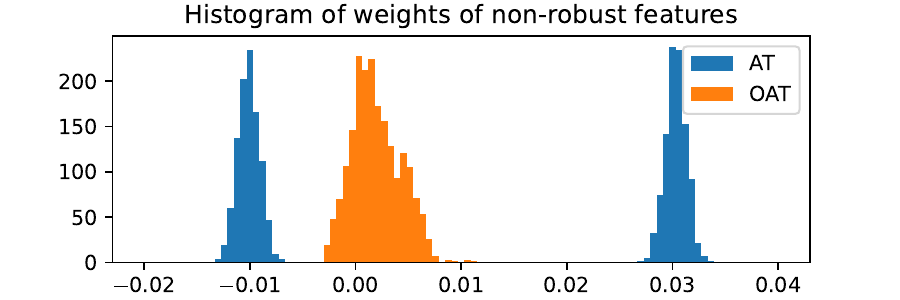}
    \caption{Weights of non-robust features at time $ t= 50$.}
    \label{fig: SVM norm}
\end{figure}

\begin{figure*}[t]
     \centering
     \begin{minipage}[b]{0.33\textwidth}
         \centering
         \includegraphics[width=\textwidth]{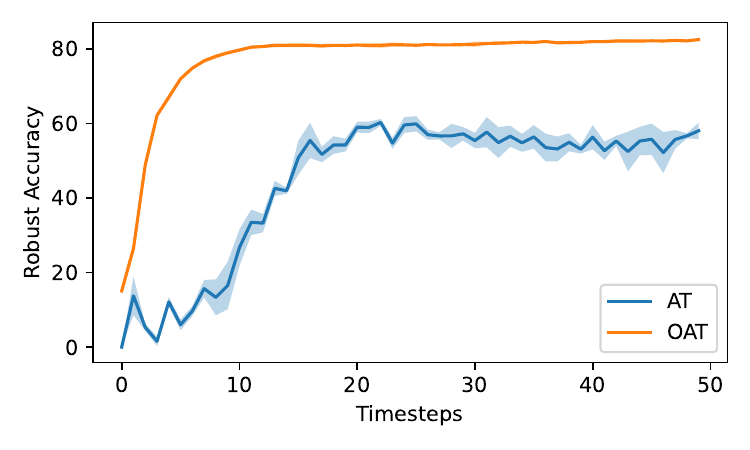}
     \end{minipage}
     \hfill
     \begin{minipage}[b]{0.33\textwidth}
         \centering
         \includegraphics[width=\textwidth]{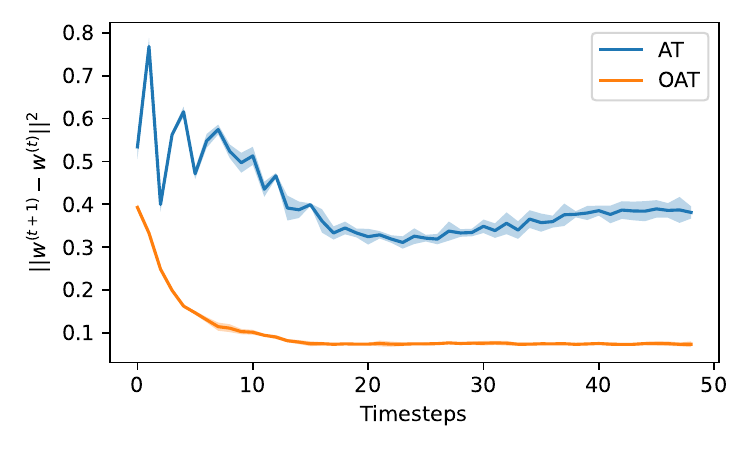}
     \end{minipage}
     \hfill
          \begin{minipage}[b]{0.33\textwidth}
         \centering
         \includegraphics[width=\textwidth]{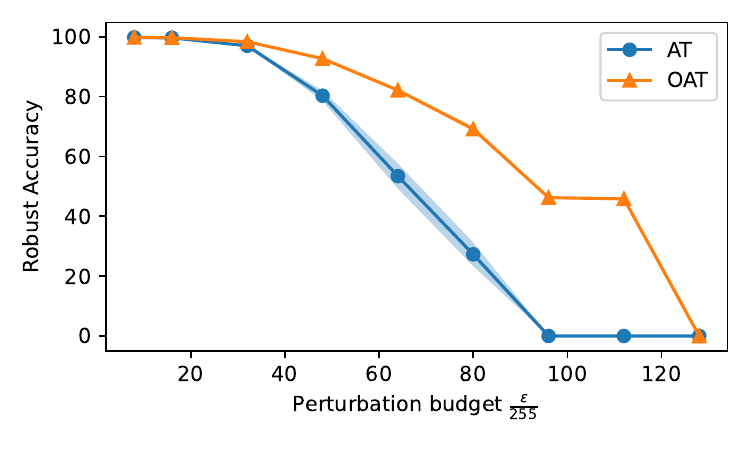}
     \end{minipage}

        \caption{Robust accuracy (left) and weight difference (mid) of AT and OAT on a binary classification task between $0,1$ on MNIST dataset when $\varepsilon = \frac{64}{255}$, and robust accuracy as we vary $\varepsilon$ from $\frac{8}{255}$ to $\frac{128}{255}$ (right). }
        \vspace{-0.1cm}
        \label{fig:1}
\end{figure*}
One explanation is Theorem \ref{lemma: adversarial training uses non-robust feature (general)}, which states that AT can lead to a model that puts non-trivial weight on non-robust features.
In our dataset, feature $j$ for $j=2,\dots,d+1$ are non-robust but predictive so that if our model relies on these features, we can have high standard accuracy but low robust accuracy. 
We can see that a model trained with AT puts more weight on non-robust features (see Figure \ref{fig: SVM norm}) and puts a higher magnitude on the positive weight which help the model to achieve $100$ percent standard accuracy.  
On the other hand, the model trained with OAT achieves $70$ percent standard accuracy and robust accuracy. The number is consistent with our construction, where we assume that the robust feature is correct with probability $0.7$. In addition, we can see that OAT leads to a model that puts a much lower weight on the non-robust features (see Figure \ref{fig: SVM norm}). This is consistent with our theoretical finding that OAT leads to a robust classifier. We note that the weights are not exactly zero because we use SGD to optimize the model.

\subsection{MNIST dataset}
\textcolor{black}{
We run experiments on a binary classification task between digits $0$ and $1$ on MNIST dataset \citep{lecun1998gradient}. The training and testing data have size $8,000$ and $1,500$, respectively. We train a linear classifier with AT and OAT for $50$ timesteps.
We use Gradient Descent with Adam optimizer and learning rate $0.01$ to update our model parameter. At each timestep, for both OAT and AT, we update the model parameter with 5 gradient steps.}

\textcolor{black}{
\textbf{Non-convergence of Advesarial Training} We report the difference between weight $||w^{(t+1)} - w^{(t)}||^2$ in Figure \ref{fig:1} (mid). The weight difference of AT fluctuates around $0.3$, almost three times the weight difference of OAT.}

\textcolor{black}{\textbf{Robustness.} We report the robust accuracy at each timestep $t$ when the perturbation budget is $\varepsilon = \frac{64}{255}$ in Figure \ref{fig:1} (left). We see that OAT leads to a higher robust accuracy and improved convergence than AT. For instance, at timestep $10$, the robust accuracy of a model trained with OAT reaches around $80 \%$ while the value for AT is at $20 \%$.
}
\begin{figure}[ht]
    \centering
    \includegraphics[width =0.7\columnwidth]{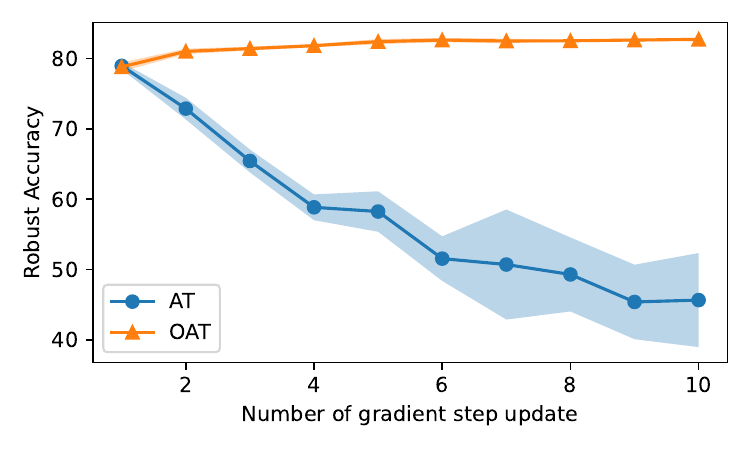}
    \caption{Robust accuracy of AT and OAT with varying number of gradient steps.}
    \label{fig: grad step}
\end{figure}

\textcolor{black}{
\textbf{Ablations.} We report robust accuracy  as we vary $\varepsilon$ from $\frac{8}{255}$ to $\frac{128}{255}$ in Figure \ref{fig:1} (right). 
We note that increasing the perturbation budget is equivalent to increasing the proportion of non-robust features. As the perturbation budget grows, robust accuracy of both model drop significantly and reach $0 \%$ when  $\varepsilon = \frac{128}{255}$. 
We observe that OAT is more resistant to large perturbation budget than AT. 
For example, when $\varepsilon = \frac{96}{255}$, a model trained with OAT has a robust accuracy around $50 \%$ while the robust accuracy is $0 \%$ for AT. In addition, we report robust accuracy as we vary the number of gradient steps we take to update the model parameter at each time step (default is $5$) when $\varepsilon = \frac{64}{255}$ in Figure \ref{fig: grad step}. We found that as we take more gradient step, the robust accuracy of AT drop significantly while the robust accuracy of OAT increases slightly. We note that more gradient steps implies that the model parameter is closer to the optimal parameter given adversarial examples at each timestep which is the setting that we studied. Surprisingly, when taking only $1$ gradient step, the robust accuracy of AT and OAT are similar but with an additional gradient step, the robust accuracy of AT drop sharply by $10\%$ and by almost $20\%$ with two additional gradient steps.
}
\section{DISCUSSION}
In this work, we study the dynamics of adversarial training from a SLAR game perspective. Our framework is general since the SLAR game does not make any assumption on the data distribution except that it has the symmetric means and the features given the label are independent of others. We find that iteratively training a model on adversarial examples does not suffice for a robust model, as the model manages to pick up signals from non-robust features at every epoch. One factor that leads to this phenomenon is a worst-case perturbation, which shifts the distribution across the decision boundary far enough so that the perturbed feature is predictive. On the other hand, we prove the existence, uniqueness, and robustness properties of a Nash equilibrium in the SLAR game. We note that this game has an infinite action space and is not a convex-concave game. Surprisingly, a perturbation function that leads to a Nash equilibrium is not the worst-case perturbation but is the one that perturbs the non-robust features to have zero means. Intuitively, this prevents the model from relying on non-robust features. In contrast of AT, the worst-case perturbation in OAT leads to a robust model. We remark that in our analysis, we assume that each player can find the optimal solution of their strategy at every iteration. This may not hold in practice since PGD or SGD are usually deployed to optimize for the optimal solution. However, our analysis serves as a foundation for future research on adversarial robustness game when the current assumption does not hold. That is, studying OAT in the regime when we do not have access to the closed-form adversarial examples e.g. neural networks, is an interesting future direction. 

\subsubsection*{Acknowledgements}
This work was supported in part by NSF grants CCF-1910321, DARPA under cooperative agreements HR00112020003, and HR00112020006, and NSERC Discovery Grant RGPIN-2022-03215, DGECR-2022-00357.


\bibliographystyle{plainnat}
\bibliography{ref.bib}

\newpage
\clearpage
\newpage
\appendix
\onecolumn
\section{PROPERTY OF THE OPTIMAL SOLUTION OF THE SVM OBJECTIVE}
We first look at the relationship between the optimal solution of this SVM objective and the underlying data distribution.

\begin{lem}[Sign of the the optimal solution]
\label{lemma: sign w}
 Let $w^* = [w^*_1, \dots, w^*_d]$ be an optimal solution of the SVM objective \eqref{eq: svm objective}.  If each feature is independent of each other, for a feature $i$ with
\begin{equation*}
    \mathbb{E}[x_i | y = -1] \leq 0 \leq \mathbb{E}[x_i | y = 1],
\end{equation*}
we have $w_i^* \geq 0$. Conversely, if 
\begin{equation*}
    \mathbb{E}[x_i | y = 1] \leq 0 \leq \mathbb{E}[x_i | y = -1],
\end{equation*}
then we have $w_i^* \leq 0$.
\end{lem}
\begin{proof}
Assume that
\begin{equation}
\label{eq: mean -1 < mean 1}
    \mathbb{E}[x_i | y = -1] \leq 0 \leq \mathbb{E}[x_i | y = 1].
\end{equation}
We will show that for an optimal weight $w^*$ with $w_i^* < 0$, we can reduce the SVM objective by setting $w_i^* = 0$. This would contradict with the optimality of $w^*$ and implies that we must have $w^*_i \geq 0$ instead. Recall that the SVM objective is given by
\begin{equation*}
    \mathbb{E}_{(x,y) \sim D}[\max(0, 1 - y\sum_{j = 1}^dw_jx_j)] + \frac{\lambda}{2}\sum_{j =1}^dw_j^2.
\end{equation*}
We denote the first term of the objective as the risk term and the second term as the regularization term. By Jensen's inequality, we know that
\begin{align*}
    &\mathbb{E}_{(x,y) \sim D}[\max(0, 1 - y\sum_{j=1}^dw_j^*x_j)] \\
    &= \mathbb{E}_{(x,y) \sim D}[\max(y\sum_{j\neq i}w_j^*x_j - 1, -yw_i^*x_i) + 1 - y\sum_{j\neq i}w_j^*x_j]\\
    &\geq \mathbb{E}_y\mathbb{E}_{x_j|y}[\max(y\sum_{j\neq i}w_j^*x_j - 1, \mathbb{E}_{x_i|y}[-yw_i^*x_i]) + 1 - y\sum_{j\neq i}w_j^*x_j].
\end{align*}
We can split the expectation between $x_i, x_j$ because each feature $i$ are independent of each other. From \eqref{eq: mean -1 < mean 1} and $w_i^* < 0$, we have
\begin{equation*}
    \mathbb{E}_{x_i|y}[-yw_i^*x_i | y = -1] = w_i^*\mathbb{E}_{x_i}[x_i | y = -1] \geq 0,
\end{equation*}
and 
\begin{equation*}
    \mathbb{E}_{x_i|y}[-yw_i^*x_i| y = 1] = -w_i^*\mathbb{E}_{x_i}[x_i | y = 1] \geq 0.
\end{equation*}
Therefore, 
\begin{equation*}
    \mathbb{E}_{x_i|y}[-yw_i^*x_i] \geq 0.
\end{equation*}
This implies that,
\begin{align*}
    &\mathbb{E}_{(x,y) \sim D}[\max(0, 1 - y\sum_{j=1}^dw_j^*x_j))]\\
    &\geq \mathbb{E}_y\mathbb{E}_{x_j|y}[\max(y\sum_{j\neq i}w_j^*x_j - 1, \mathbb{E}_{x_i|y}[-yw_i^*x_i]))+ 1 - y\sum_{j\neq i}w_j^*x_j] \\
    &\geq \mathbb{E}_y\mathbb{E}_{x_j|y}[\max(y\sum_{j\neq i}w_j^*x_j - 1, 0)) + 1 - y\sum_{j\neq i}w_j^*x_j] \\
    &= \mathbb{E}_y\mathbb{E}_{x_j|y}[\max(0,1 - y\sum_{j\neq i}w_j^*x_j)].
\end{align*}
The risk term in the SVM objective when $w_i^* < 0$ is no better than when $w_i^* = 0$. However, the regularization term is higher. 
\begin{equation*}
    \frac{\lambda}{2}\sum_{j=1}^d (w_j^*)^2 > \frac{\lambda}{2}\sum_{j \neq i} (w_j^*)^2.
\end{equation*}
Therefore, we can reduce the SVM objective by setting $w_i^* = 0$. Therefore, $w_i < 0$ can't be the optimal weight and we must have $w_i^* \geq 0$.  Similarly, we can apply the same idea to the other case.
\end{proof}

\begin{corr}
\label{lemma: w_i = 0 when mu_i = 0}
Let $w^* = [w^*_1, \dots, w^*_d]$ be an optimal solution of the SVM objective \eqref{eq: svm objective}.  If each feature are independent of each other, for a feature $i$ with
\begin{equation*}
    \mathbb{E}[x_i | y = -1] = 0 = \mathbb{E}[x_i | y = 1],
\end{equation*}
then we have $w_i^* = 0$.
\end{corr}

\begin{lem}[Upper bound on the magnitude]
\label{lemma: magnitude of w*}
 Let $w^* = [w^*_1, \dots, w^*_d]$ be an optimal solution of the SVM objective \eqref{eq: svm objective}. We have
\begin{equation*}
    ||w^*||_2 \leq \sqrt{\frac{2}{\lambda}}.
\end{equation*}
\end{lem}
\begin{proof}
Since, $w^*$ is an optimal solution of \eqref{eq: svm objective}, we have
\begin{equation*}
    \mathcal{L}( w^*) \leq \mathcal{L}( 0) = 1.
\end{equation*}
Since
\begin{equation*}
    \mathcal{L}( w^*) \geq 0 + \frac{\lambda}{2}||w^*||^2_2,
\end{equation*}
we have
\begin{align*}
    \frac{\lambda}{2}||w^*||^2_2 \leq 1 \\
    ||w^*||_2 \leq \sqrt{\frac{2}{\lambda}}.
\end{align*}
\end{proof}

\begin{lem}
\label{lemma: bound of E(max(0,X))}
Let $X$ be a random variable with mean $\mu$ and variance $\sigma^2$ then we have
\begin{equation*}
    \max(0, \mathbb{E}[X]) \leq \mathbb{E}[\max(0,X)] \leq \max(0, \mathbb{E}[X]) + \frac{1}{2}\sqrt{\operatorname{Var}(X)}.
\end{equation*}
\end{lem}
\begin{proof}
First, we know that $\max(0,x)$ is convex and by Jensen's inequality we have
\begin{align*}
    \mathbb{E}[\max(0,X)] \geq \max(0, \mathbb{E}[X]).
\end{align*}
We also know that $x^2$ is convex, by Jensen's inequality, we have
\begin{align*}
    \mathbb{E}[X^2] \geq \ \mathbb{E}[|X|]^2.
\end{align*}
Next, we observe that $\max(0,x) = \frac{1}{2}( x + |x|)$,
\begin{align*}
    \mathbb{E}[\max(0,X)] &=  \mathbb{E}[\frac{1}{2}(X + |X|)]\\
    &\leq \frac{1}{2}(\mathbb{E}[(X)] + \sqrt{\mathbb{E}[X^2]})\\
    &= \frac{1}{2}(\mathbb{E}[(X)] + \sqrt{\operatorname{Var}(X) + \mathbb{E}[X]^2 })\\
    &\leq \frac{1}{2}(\mathbb{E}[(X)] + \sqrt{\operatorname{Var}(X)} + |\mathbb{E}[X]|\\
    &= \max(0, \mathbb{E}[X]) + \frac{1}{2}\sqrt{\operatorname{Var}(X)}.
\end{align*}
\end{proof}

\begin{lem}[Lower bound on the magnitude]
\label{lem: lower bound on the magnitude}
 Let $w^* = [w^*_1, \dots, w^*_d]$ be an optimal solution of the SVM objective \eqref{eq: svm objective}. Let each feature $x_i$ has mean and variance as follows
 \begin{equation*}
     \mathbb{E}[x_i|y] = y\mu_i, \operatorname{Var}(x_i|y) = \sigma_i^2,
 \end{equation*}
and the feature are independent of each other then
\begin{equation*}
    ||w^*||_2 \geq \frac{1}{||\mu||_2}(1 - \frac{1}{2}(\frac{\Bar{\sigma}_\mu}{||\mu||_2} + \frac{\lambda}{2||\mu||_2^2})),
\end{equation*}
where
\begin{equation*}
    \mu = [\mu_1, \mu_2, \dots, \mu_d],\ \Bar{\sigma}_\mu^2 = \frac{\sum_{i=1}^d \mu_i^2\sigma_i^2}{\sum_{i=1}^d \mu_i^2}.
\end{equation*}
\end{lem}
\begin{proof}
We will prove this by contradiction. Let $w^*$ be an optimal solution with
\begin{equation}
\label{eq: condition upper bound ||w||}
    ||w^*||_2 < \frac{1}{||\mu||_2}(1 - \frac{1}{2}(\frac{\Bar{\sigma}_\mu}{||\mu||_2} + \frac{\lambda}{2||\mu||_2^2})).
\end{equation}
Rearrange to
\begin{equation*}
1 - ||w^*||_2 ||\mu||_2 > \frac{1}{2}(\frac{\Bar{\sigma}_\mu}{||\mu||_2} + \frac{\lambda}{2||\mu||_2^2}) > 0.
\end{equation*}
The SVM objective is given by
\begin{align*}
    \mathcal{L}(w^*) &= \mathbb{E}[\max(0, 1 - y \sum_{j=1}^d w_j^*x_j)] + \frac{\lambda}{2}||w^*||_2^2 \\
    &\geq \max(0, \mathbb{E}[1 - y \sum_{j=1}^d w_j^*x_j])\\
    &= \max(0, 1 - \sum_{j=1}^d w_j^*\mu_j)\\
    &\geq \max(0, 1 - ||w^*||_2||\mu||_2)\\
    &= 1 - ||w^*||_2||\mu||_2\\
    &>  \frac{1}{2}(\frac{\Bar{\sigma}_\mu}{||\mu||_2} + \frac{\lambda}{2||\mu||_2^2}). 
\end{align*}
Here we apply Lemma \ref{lemma: bound of E(max(0,X))}, the Cauchy–Schwarz inequality and \eqref{eq: condition upper bound ||w||} respectively. On the other hand, consider $w' = \frac{\mu}{||\mu||_2^2}$. We have
\begin{equation*}
    ||w'||_2 = \frac{1}{||\mu||_2} > ||w^*||_2.
\end{equation*}
 From Lemma \ref{lemma: bound of E(max(0,X))}, the SVM objective satisfies
\begin{align*}
    \mathcal{L}(w') &= \mathbb{E}[\max(0, 1 - y \sum_{j=1}^d w_j'x_j)] + \frac{\lambda}{2}||w'||_2^2\\
    &\leq  \max(0, \mathbb{E}[1 - y \sum_{j=1}^d w'_jx_j]) + \frac{1}{2}\sqrt{ \operatorname{Var}(1 - y \sum_{j=1}^d w_j'x_j)} + \frac{\lambda}{2}||w'||_2^2 \\
    &= \max(0, \mathbb{E}[1 - \frac{||\mu||_2}{||\mu||_2}]) + \frac{1}{2}\sqrt{ \sum_{j=1}^d (w_j')^2\operatorname{Var}(x_j)} + \frac{\lambda}{2}||w'||_2^2\\
    &= \frac{1}{2}\sqrt{\frac{\sum_{j=1}^d \mu_j^2\sigma_j^2}{||\mu||_2^4}} + \frac{\lambda}{2}\frac{1}{||\mu||_2^2}\\
    &= \frac{1}{2}(\frac{\Bar{\sigma}_\mu}{||\mu||_2} + \frac{\lambda}{2||\mu||_2^2})\\
    &< \mathcal{L}(w^*).
\end{align*}
This contradicts with the optimality of $w^*$.
\end{proof}

\begin{lem}
\label{lem: unique solution of SVM}
Let $u, v$ be optimal solution of the SVM objective \eqref{eq: svm objective} under a data distribution $\mathcal{D}$ then we must have $u = v$.
\end{lem}
\begin{proof}
We will prove this by contradiction, assume that $u \neq v$. Since both are optimal solutions, we have
\begin{align*}
     \mathcal{L}(u) = \mathcal{L}(v)
\end{align*}
when
\begin{align*}
         \mathcal{L}(u) = \mathbb{E}_{(x,y) \sim \mathcal{D}}[\max(0, 1 - yu^{\top}x)] + \frac{\lambda}{2}||u||^2_2\\
     \mathcal{L}(v) = \mathbb{E}_{(x,y) \sim \mathcal{D}}[\max(0, 1 - yv^{\top}x)] + \frac{\lambda}{2}||v||^2_2.\\
\end{align*}
Consider
\begin{align*}
    \mathcal{L}(u) 
    &= \frac{1}{2}(\mathcal{L}(u) + \mathcal{L}(v) )\\
    &=  \frac{1}{2}(\mathbb{E}_{(x,y) \sim \mathcal{D}}[\max(0, 1 - yu^{\top}x)] + \frac{\lambda}{2}||u||^2_2 \\
    &\quad + \mathbb{E}_{(x,y) \sim \mathcal{D}}[\max(0, 1 - yv^{\top}x)] + \frac{\lambda}{2}||v||^2_2)\\
    &\geq \frac{1}{2}(\mathbb{E}_{(x,y) \sim \mathcal{D}}[\max(0, 2 - y(u+v)^{\top}x)] + \frac{\lambda}{2}(||u||^2_2 + ||v||^2_2))\\
    &> \frac{1}{2}(\mathbb{E}_{(x,y) \sim \mathcal{D}}[\max(0, 2 - y(u+v)^{\top}x)] + \frac{\lambda}{2}(2||\frac{u + v}{2}||^2_2 )\\
    &= \mathbb{E}_{(x,y) \sim \mathcal{D}}[\max(0, 1 - y\frac{(u+v)^{\top}}{2}x)] + \frac{\lambda}{2}||\frac{u + v}{2}||^2_2 \\
    &= \mathcal{L}(\frac{u+v}{2}). 
\end{align*}
We utilize an inequality 
\begin{equation*}
    \max(0,a) + \max(0,b) \geq \max(0, a+b)
\end{equation*}
and
\begin{equation*}
    ||u||^2_2 + ||v||^2_2 \geq 2||\frac{u + v}{2}||^2_2.
\end{equation*}
The equality of the second inequality does not hold since $u \neq v$ so we have
\begin{equation*}
    ||u||^2_2 + ||v||^2_2 > 2||\frac{u + v}{2}||^2_2.
\end{equation*}
This contradicts with the optimality of $u,v$ as $\frac{u+v}{2}$ leads to a lower objective. Therefore, we must have $u = v$.
\end{proof}

\section{DISCUSSION ON NON-ROBUST FEATURES}\label{appendix: non-robust feature}
We would like to point out that there is an alternative definition of robust features.
\begin{definition}[Non-robust feature (alternative)] We say that feature $x_i$ is non-robust (alternative) if there exists a perturbation $\delta_{i} : \mathcal{D} \to \mathcal{B}(\varepsilon)$ such that after adding the perturbation, the distribution of feature $i$ of class $y = -1$ is close to the distribution of feature $i$ of class $y = 1$. Formally, let $D_{i,y}+\delta$ be the distribution of the perturbed feature $x_i + \delta_{i}(x,y)$ for a class $y$. The feature $x_i$ is non robust when
\begin{equation*}
    \operatorname{Dist}(D_{i,-1}+\delta, D_{i,1}+\delta) \leq c,
\end{equation*}
where $\operatorname{Dist}$ is your choice of distance metric (this can be Total variation, KL divergence or etc.) and $c$ is a constant. Otherwise, the feature $x_i$ is robust (alternative).
\end{definition}

\begin{figure}[tph]
 \hfill
    \begin{subfigure}{0.48\textwidth}
        \centering
        \includegraphics[width =\textwidth]{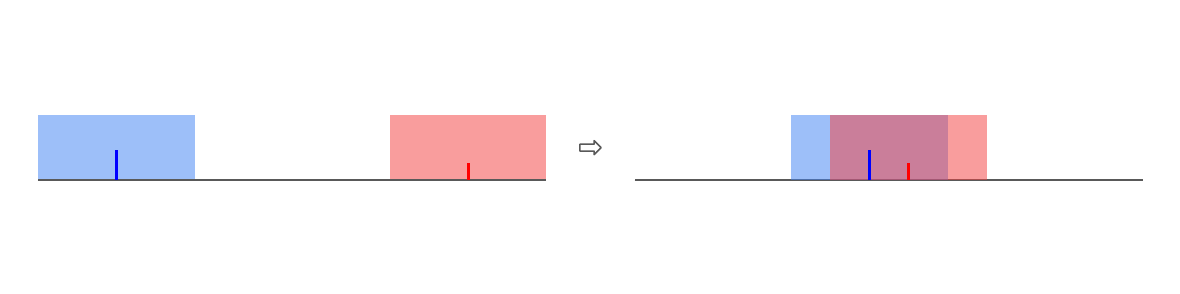}
        \caption{An example of a feature that is only robust to mean shift but is not robust}
        \label{fig: robust only}
     \end{subfigure}
     \begin{subfigure}{0.48\textwidth}
    \centering
        \includegraphics[width =\textwidth]{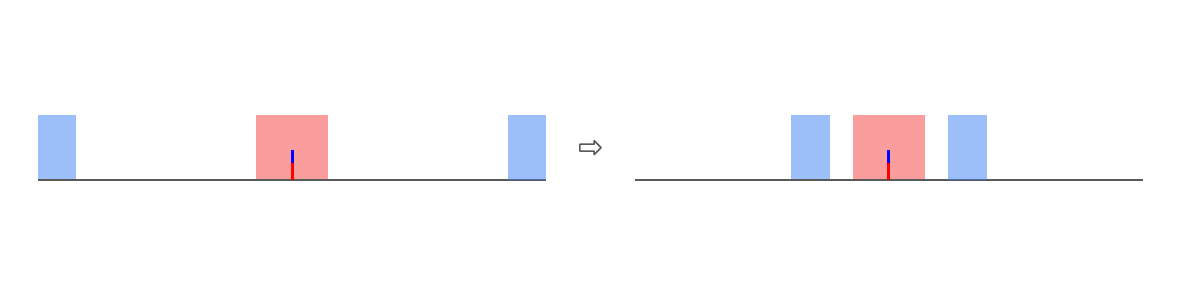}
        \caption{An example of a feature that is robust but is not robust to the mean shift because two distributions have the same mean}
        \label{fig: robust mean shift only}
     \end{subfigure}
     \hfill
    \caption{Features that are robust to mean shift are not necessarily robust in the alternative definition. Red and blue represents each class $y = -1, 1$ and a line represents the mean of each class.}
    \label{fig: robust vs robust alt}
\end{figure}

We note that this alternative definition of non-robust feature is different from one we defined earlier which we refer to non-robust to mean shift. Figure \ref{fig: robust vs robust alt} provides examples for this where each color represent the distribution of the feature given class $y = -1, 1$.

 \begin{enumerate}
     \item We can have two distributions where the distance between the means is $2.1\varepsilon$ that is the feature is robust to mean shift. However when perturbed, the two distributions are almost aligned with each other, thus not robust (alternative) (Figure \ref{fig: robust only}).
     \item On the other hand, we can have two distributions with the same mean such that the feature is not robust to mean shift. However, the shape of these distributions are different enough so that it is not possible to perturb them to be close to each other. Therefore, this feature is robust (alternative) (Figure \ref{fig: robust mean shift only}).
 \end{enumerate}

\section{STANDARD TRAINING RELIES ON NON-ROBUST FEATURES}\label{appendix: standard training relies on non-robust}
\begin{theorem}[Standard training uses non-robust feature]
\label{lemma: standard training uses non-robust feature (general)}
 Let the data distribution follows
the distribution as in Definition \ref{def: distribution for AT}. Let $w^* = [w^*_1,w^*_2, \dots, w^*_{d+1}]$ be the optimal solution under a standard SVM objective, 
\begin{equation*}
    w^* = \argmin_w  \mathbb{E}[\max(0, 1 - yw^{\top}x )] + \frac{\lambda}{2}||w||^2_2.
\end{equation*}
If
\begin{equation}
    \label{eq; condition for p to use non-robust feature (general)}
    p < 1 - \frac{1}{2}(\frac{\Bar{\sigma}_\mu}{||\mu||_2} + \frac{\lambda}{2||\mu||_2^2}) - \frac{1}{2}\sqrt{\frac{2}{\lambda}}\Bar{\sigma}_\mu,
\end{equation}
where
\begin{equation*}
   \mu = [\mu_1, \mu_2, \dots, \mu_{d+1}],\ \Bar{\sigma}_\mu = \sqrt{\frac{\sum_{j=1}^{d+1}\mu_j^2\sigma_j^2}{||\mu||^2_2}},
\end{equation*}
then $w^*$  will rely on the non-robust feature $j$,
\begin{equation*}
    w_1^* \leq \sum_{j \geq 2}w^*_j\mu_j.
\end{equation*}
This also implies that
\begin{equation*}
    \sum_{j=2}^{d+1} (w^*_j)^2 \geq \frac{||w^*||_2^2}{1 + \sum_{j=2}^{d+1} \mu_j^2}.
\end{equation*}
\end{theorem}
\begin{proof}
We will prove by contradiction. Assume that $w^*$ is an optimal solution of the SVM objective and \eqref{eq; condition for p to use non-robust feature (general)} holds, and
\begin{equation*}
    w^*_1 >\sum_{j = 2}^{d+1}w^*_j\mu_j .
\end{equation*}
The SVM objective is given by
\begin{align*}
    \mathcal{L}(w^*) &= \mathbb{E}[\max(0, 1 - \sum_{i=1}^{d+1}yw^*_ix_i )] + \frac{\lambda}{2}||w^*||^2_2\\
    &= p\mathbb{E}[\max(0, 1 - w_1^* - \sum_{j=2}^{d+1}yw^*_jx_j )]  +(1-p)\mathbb{E}[\max(0, 1 + w_1^* - \sum_{j=2}^{d+1}yw^*_jx_j )] + \frac{\lambda}{2}||w^*||^2_2\\
    &\geq (1-p)\mathbb{E}[\max(0, 1 + w_1^* - \sum_{j=2}^{d+1}yw^*_jx_j )] + \frac{\lambda}{2}||w^*||^2_2.\\
\end{align*}
From Lemma \ref{lemma: bound of E(max(0,X))}, $ \mathbb{E}[\max(0,X)] \geq \max(0, \mathbb{E}[X])$. Therefore, 
\begin{align*}
\mathcal{L}(w^*) &\geq (1-p)\max(0,\mathbb{E}[1 + w_1^* - \sum_{j=2}^{d+1}yw^*_jx_j]))+ \frac{\lambda}{2}||w^*||^2_2\\
&= (1-p)\max(0,1 + w_1^* - \sum_{j=2}^{d+1}w^*_j\mu_j))+ \frac{\lambda}{2}||w^*||^2_2\\
&\geq (1-p) + \frac{\lambda}{2}||w^*||^2_2.
\end{align*}
The last inequality holds because
\begin{equation*}
    w^*_1 > \sum_{j = 2}^{d+1}w^*_j\mu_j.
\end{equation*}
Now, consider $w' = \frac{||w^*||_2}{||\mu||_2}\mu$ when $\mu = [\mu_1, \mu_2,\dots, \mu_{d+1}]$. We have
\begin{equation*}
    ||w'||_2 = \frac{||w^*||_2}{||\mu||_2}||\mu||_2 = ||w^*||_2,
\end{equation*}
and from Lemma \ref{lemma: bound of E(max(0,X))},
\begin{equation*}
    \mathbb{E}[\max(0,X)] \leq \max(0, \mathbb{E}[X]) + \frac{1}{2}\sqrt{\operatorname{Var}(X)}.
\end{equation*}
we have
\begin{align*}
    &\mathcal{L}(w') = \mathbb{E}[\max(0, 1 - \sum_{i=1}^{d+1}yw'_ix_i )] + \frac{\lambda}{2}||w'||^2_2\\
    &\leq  \max(0, \mathbb{E}[1 - y \sum_{j=1}^{d+1} w'_jx_j])  + \frac{1}{2}\sqrt{ \operatorname{Var}(1 - y \sum_{j=1}^{d+1} w_j'x_j)} + \frac{\lambda}{2}||w^*||_2^2 \\
    &= \max(0, 1 - ||w^*||_2||\mu||_2) + \frac{1}{2}\sqrt{ \sum_{j=1}^{d+1} (w_j')^2\operatorname{Var}(x_j)} + \frac{\lambda}{2}||w^*||_2^2\\
    &= \max(0, 1 - ||w^*||_2||\mu||_2)  + \frac{1}{2}||w^*||_2\sqrt{ \frac{\sum_{j=1}^{d+1} \mu_j^2\sigma_j^2}{||\mu||_2^2}} + \frac{\lambda}{2}||w^*||_2^2\\
    &=  \max(0, 1 - ||w^*||_2||\mu||_2) + \frac{1}{2}||w^*||_2\Bar{\sigma_\mu} + \frac{\lambda}{2}||w^*||_2^2.
\end{align*}
From Lemma \ref{lem: lower bound on the magnitude}
\begin{equation*}
    ||w^*||_2 \geq \frac{1}{||\mu||_2}(1 - \frac{1}{2}(\frac{\Bar{\sigma}_\mu}{||\mu||_2} + \frac{\lambda}{2||\mu||_2^2})),
\end{equation*}
and Lemma \ref{lemma: magnitude of w*}
\begin{equation*}
    ||w^*||_2 \leq \sqrt{\frac{2}{\lambda}},
\end{equation*}
we have the upper bound of $\mathcal{L}(w')$ as follows
\begin{align*}
    \mathcal{L}(w') &\leq \frac{1}{2}(\frac{\Bar{\sigma}_\mu}{||\mu||_2} + \frac{\lambda}{2||\mu||_2^2}) + \frac{1}{2}\sqrt{\frac{2}{\lambda}}\Bar{\sigma_\mu} + \frac{\lambda}{2}||w^*||_2^2.
\end{align*}
From \eqref{eq; condition for p to use non-robust feature (general)} we know that
\begin{align*}
    p &< 1 - \frac{1}{2}(\frac{\Bar{\sigma}_\mu}{||\mu||_2} + \frac{\lambda}{2||\mu||_2^2}) - \frac{1}{2}\sqrt{\frac{2}{\lambda}}\Bar{\sigma}_\mu.
\end{align*}
So
\begin{equation*}
     1-p > \frac{1}{2}(\frac{\Bar{\sigma}_\mu}{||\mu||_2} + \frac{\lambda}{2||\mu||_2^2}) + \frac{1}{2}\sqrt{\frac{2}{\lambda}}\Bar{\sigma}_{\mu},
\end{equation*}
and that
\begin{equation*}
    \mathcal{L}(w^*) < \mathcal{L}(w').
\end{equation*}
This contradicts with the fact that $w^*$ is an optimal solution. Therefore, if $w^*$ is an optimal solution and \eqref{eq; condition for p to use non-robust feature (general)} holds, we must have
\begin{equation*}
    w^*_1 \leq \sum_{j = 2}^{d+1}w^*_j\mu_j.
\end{equation*}
By Cauchy-Schwarz inequality, we have
\begin{align*}
    &\qquad (w_1^*)^2 \leq (\sum_{j = 2}^{d+1}w^*_j\mu_j)^2 \leq (\sum_{j = 2}^{d+1}(w^*_j)^2)(\sum_{j = 2}^{d+1}\mu_j^2)\\
    &\Longleftrightarrow\sum_{i = 1}^{d+1}(w^*_i)^2 \leq (\sum_{j = 2}^{d+1}(w^*_j)^2)(1 + \sum_{j = 2}^{d+1}\mu_j^2)\\
    &\Longleftrightarrow \frac{||w^*||^2_2}{1 + \sum_{j = 2}^{d+1}\mu_j^2} \leq \sum_{j = 2}^{d+1}(w^*_j)^2.
\end{align*}
\end{proof}

The condition in Theorem \ref{lemma: standard training uses non-robust feature (general)} holds when the number of non-robust features $d$ and the regularization parameter $\lambda$ are large enough, 
\begin{equation*}
    p < 1 - \frac{1}{2}(\frac{\Bar{\sigma}_\mu}{||\mu||_2} + \frac{\lambda}{2||\mu||_2^2}) - \frac{1}{2}\sqrt{\frac{2}{\lambda}}\Bar{\sigma}_\mu.
\end{equation*}
When $||\mu||_2$ is large the RHS will be larger so the condition holds for more value of $p$. In an extreme case when $||\mu||_2 \to \infty$, we have
\begin{equation*}
    \frac{1}{2}(\frac{\Bar{\sigma}_\mu}{||\mu||_2} + \frac{\lambda}{2||\mu||_2^2}) \to 0.
\end{equation*}
The condition becomes
\begin{equation*}
    p < 1 -  \frac{1}{2}\sqrt{\frac{2}{\lambda}}\Bar{\sigma}_\mu.
\end{equation*}
Therefore, a necessary condition is that the regularization parameter has to be large enough
\begin{equation*}
    \lambda > \frac{\Bar{\sigma}_\mu^2}{2(1-p)^2}
\end{equation*}
where $\lambda$ scales with the weighted average of the variance $\Bar{\sigma}_\mu^2$. In general, the term $||\mu||_2 = \sqrt{\sum_{j=1}^d \mu_j^2}$ depends on the magnitude and the number of non-robust features. If the each $\mu_i$ is large then we only need a smaller $d$ for $||\mu||_2$  to be large.\\

We note that the terms $\mu_1, \sigma_1^2$ are fuctions of $p$. However, we can bound them with constants
\begin{equation*}
    \mu_1 = \mathbb{E}[x_1y] = 2p-1 \leq 1
\end{equation*}
and 
\begin{equation*}
    \sigma_1^2 = 1 - (2p-1)^2\leq 1.
\end{equation*}
In addition, the contribution of $\mu_1, \sigma_1$ would be in $\mathcal{O}(\frac{1}{d})$ as we have a large number of non-robust feature $d$.

\section{ADVERSARIAL TRAINING RELIES ON NON-ROBUST FEATURE}\label{appendix: adversarial training relies on non-robust}
\begin{theorem}[Adversarial training uses non-robust feature]
\label{lemma: adversarial training uses non-robust feature (general)}
Let the data distribution follows
the distribution as in Definition \ref{def: distribution for AT}. Let $\delta$ be a perturbation given by adversarial training with a perturbation budget $\varepsilon$. We assume that the perturbation is in the form of the worst case perturbation where 
\begin{equation*}
    \delta(x,y) \in \{ -y\varepsilon, 0 , y\varepsilon \}^{d+1}.
\end{equation*}
Let $w^* = [w^*_1,w^*_2, \dots, w^*_{d+1}]$ be the optimal solution under a standard SVM objective on the perturbed data $x+\delta$, 
\begin{equation*}
    w^* = \argmin_w  \mathbb{E}[\max(0, 1 - yw^{\top}(x + \delta(x,y)) )] + \frac{\lambda}{2}||w||^2_2.
\end{equation*}
If
\begin{equation}
\label{eq: condition for p to use non-robust feature AT}
    p < 1 - \sup_s((\frac{1}{2}(\frac{\Bar{\sigma}_{\mu,s}}{||\mu + \varepsilon s||_2} + \frac{\lambda}{2||\mu + \varepsilon s||_2^2}) + \frac{1}{2}\sqrt{\frac{2}{\lambda}}\Bar{\sigma}_{\mu,s}),
\end{equation}
when
\begin{equation*}
    s \in \{-1,0,1\}^{d+1}, \Bar{\sigma}_{\mu, s} = \sqrt{ \frac{\sum_{j=1}^{d+1} (\mu_j + \varepsilon s_j)^2\sigma_j^2}{||\mu + \varepsilon s||_2^2}},
\end{equation*}
then $w^*$ satisfies
\begin{equation*}
    w^*_1(1-\varepsilon) \leq \sum_{j = 2}^{d+1}w^*_j|\mu_j + \varepsilon|.
\end{equation*}
This implies
\begin{equation*}
    \sum_{j=2}^{d+1} (w^*_j)^2 \geq \frac{||w^*||_2^2(1-\varepsilon)^2}{(1-\varepsilon)^2 + \sum_{j=2}^{d+1} (\mu_j + \varepsilon)^2}.
\end{equation*}
\end{theorem}
\begin{proof}
 Let
\begin{equation*}
    \delta(x,y) = [\delta_1(x,y), \delta_2(x,y), \dots, \delta_{d+1}(x,y)],
\end{equation*}
when
\begin{equation*}
        \delta_i(x,y) = y\varepsilon s_i,
\end{equation*}
where $s_i$ can take 3 possible values
\begin{equation*}
        s_i =\left\{\begin{array}{cl}
1; \\
0;  \\
-1.
\end{array}\right.
\end{equation*}
The perturbation from adversarial training does not depends on $x$. We can see this as shifting the whole distribution for each feature. For the first feature
\begin{equation*}
    x_{1}+ \delta_1(x,y)=\left\{\begin{array}{ll}
y(1 + \varepsilon s_1) , & \text { w.p. } p;\\
-y(1 - \varepsilon s_1), & \text { w.p. } 1-p.
\end{array}\right.
\end{equation*}

For each feature $j$ for $j = 2,\dots, d+1$, this perturbation will only change the mean of the perturbed data but will preserve the variance. 
\begin{equation*}
    \mathbb{E}[x_j + \delta_j(x,y)] = y\mu_j + y\varepsilon s_j, \operatorname{Var}(x_j + \delta_j(y)) = \sigma_j^2.
\end{equation*}
We refer $\delta(y,s)$ to the perturbation where $\delta(x,y) = y\varepsilon s$. Denote the SVM objective on the data with perturbation $\delta$ as
\begin{align*}
    \mathcal{L}(w, \delta) &= \mathbb{E}[\max(0, 1 - yw^{\top}(x+\delta(x,y)) )] + \frac{\lambda}{2}||w||^2_2 \\
    &= \mathbb{E}[\max(0, 1 - \sum_{i=1}^{d+1}yw_i(x_i + \delta_i(x,y)) )] + \frac{\lambda}{2}||w||^2_2.
\end{align*}
For a fixed $s$, let $w^*$ be an optimal solution of the SVM objective on the perturbed data $x+ \delta(y,s)$ and assume that 
\begin{equation*}
    w^*_1(1-\varepsilon) >\sum_{j = 2}^{d+1}w^*_j|\mu_j + \varepsilon|.
\end{equation*}
First,
\begin{align*}
\mathcal{L}(w^*, \delta) 
&\geq (1-p)\max(0,\mathbb{E}[1 + w_1(1 - \varepsilon s_1) - \sum_{j=2}^{d+1}yw^*_j(x_j + y\varepsilon s_j)]))+ \frac{\lambda}{2}||w^*||_2^2\\
&\geq (1-p)\max(0,1 + w_1(1 - \varepsilon s_1) - \sum_{j=2}^{d+1}w^*_j(\mu_j + \varepsilon s_j))+ \frac{\lambda}{2}||w^*||_2^2\\
&\geq (1-p)\max(0,1 + w_1(1 - \varepsilon) - \sum_{j=2}^{d+1}w^*_j|\mu_j + \varepsilon|)+ \frac{\lambda}{2}||w^*||_2^2\\
&\geq 1-p + \frac{\lambda}{2}||w^*||_2^2.
\end{align*}

The last inequality holds because
\begin{equation*}
    w^*_1(1-\varepsilon) >\sum_{j = 2}^{d+1}w^*_j|\mu_j + \varepsilon|.
\end{equation*}
On the other hand, consider
\begin{equation*}
    w' = \frac{||w^*||_2}{||\mu + \varepsilon s||_2}[\mu_1 + \varepsilon s_1, \dots, \mu_{d+1} + \varepsilon s_{d+1}].
\end{equation*}
we have
\begin{equation*}
    ||w'||_2 = ||w^*||_2.
\end{equation*}
Consider
\begin{align*}
    \mathcal{L}(w',\delta) &= \mathbb{E}[\max(0, 1 - \sum_{j=1}^{d+1}yw'_j(x_j + y\varepsilon s_j) )] + \frac{\lambda}{2}||w'||^2_2\\
     &\leq  \max(0, \mathbb{E}[1 - y \sum_{j=1}^{d+1} w'_j(x_j + y\varepsilon s_j)]) + \frac{1}{2}\sqrt{ \operatorname{Var}(1 - y \sum_{j=1}^{d+1} w_j'(x_j + y\varepsilon s_j))} + \frac{\lambda}{2}||w^*||_2^2 \\
    &= \max(0, 1 - ||w^*||_2||\mu + \varepsilon s||_2) + \frac{1}{2}||w^*||_2\sqrt{ \frac{\sum_{j=1}^{d+1} (\mu_j + \varepsilon s_j)^2\sigma_j^2}{||\mu + \varepsilon s||_2^2}} + \frac{\lambda}{2}||w^*||_2^2 \\
    &= \max(0, 1 - ||w^*||_2||\mu + \varepsilon s||_2) + \frac{1}{2}||w^*||_2\Bar{\sigma}_{\mu, s} + \frac{\lambda}{2}||w^*||_2^2, \\
\end{align*}
when
\begin{equation*}
    \Bar{\sigma}_{\mu, s} = \sqrt{ \frac{\sum_{j=1}^{d+1} (\mu_j + \varepsilon s_j)^2\sigma_j^2}{||\mu + \varepsilon s||_2^2}}.
\end{equation*}
From Lemma \ref{lem: lower bound on the magnitude}
\begin{equation*}
    ||w^*||_2 \geq \frac{1}{||\mu+ \varepsilon s||_2}(1 - \frac{1}{2}(\frac{\Bar{\sigma}_{\mu,s}}{||\mu+ \varepsilon s||_2} + \frac{\lambda}{2||\mu+ \varepsilon s||_2^2})),
\end{equation*}
and Lemma \ref{lemma: magnitude of w*}
\begin{equation*}
    ||w^*||_2 \leq \sqrt{\frac{2}{\lambda}},
\end{equation*}
we have the upper bound of $\mathcal{L}(w')$ as follows
\begin{align*}
    \mathcal{L}(w', \delta) &\leq \frac{1}{2}(\frac{\Bar{\sigma}_{\mu,s}}{||\mu+ \varepsilon s||_2} + \frac{\lambda}{2||\mu+ \varepsilon s||_2^2}) + \frac{1}{2}\sqrt{\frac{2}{\lambda}}\Bar{\sigma}_{\mu,s} + \frac{\lambda}{2}||w^*||_2^2.
\end{align*}
Recall that we have
\begin{align*}
    \mathcal{L}(w^*, \delta)\geq 1-p + \frac{\lambda}{2}||w^*||_2^2.
\end{align*}
Therefore, if
\begin{align*}
    p < 1 - \frac{1}{2}(\frac{\Bar{\sigma}_{\mu,s}}{||\mu + \varepsilon s||_2} + \frac{\lambda}{2||\mu + \varepsilon s||_2^2}) - \frac{1}{2}\sqrt{\frac{2}{\lambda}}\Bar{\sigma}_{\mu,s},
\end{align*}
we would have 
\begin{equation*}
    \mathcal{L}(w', \delta) < \mathcal{L}(w^*, \delta),
\end{equation*}
which lead to a contradiction with the fact that $w^*$ is an optimal solution. This implies that for a fixed perturbation $\delta(s)$, if 
\begin{align*}
    p < 1 - \frac{1}{2}(\frac{\Bar{\sigma}_{\mu,s}}{||\mu + \varepsilon s||_2} + \frac{\lambda}{2||\mu + \varepsilon s||_2^2}) - \frac{1}{2}\sqrt{\frac{2}{\lambda}}\Bar{\sigma}_{\mu,s},
\end{align*}
then the optimal solution of the SVM objective on the perturbed data $x+ \delta(s)$ satisfies
\begin{equation*}
    w^*_1(1-\varepsilon) \leq \sum_{j = 2}^{d+1}w^*_j|\mu_j + \varepsilon|.
\end{equation*}
Now, if we have
\begin{align*}
    p < 1 - \sup_s((\frac{1}{2}(\frac{\Bar{\sigma}_{\mu,s}}{||\mu + \varepsilon s||_2} + \frac{\lambda}{2||\mu + \varepsilon s||_2^2}) + \frac{1}{2}\sqrt{\frac{2}{\lambda}}\Bar{\sigma}_{\mu,s}),
\end{align*}
we can conclude that for any perturbation $s \in \{-1,0,1\}^{d+1}$, the optimal solution of the SVM objective on the perturbed data $x+ \delta(s)$ satisfies
\begin{equation*}
    w^*_1(1-\varepsilon) \leq \sum_{j = 2}^{d+1}w^*_j|\mu_j + \varepsilon|.
\end{equation*}
Moreover, we can apply Cauchy-Schwarz inequality to have
\begin{align*}
    (w^*_1)^2 &\leq \frac{(\sum_{j = 2}^{d+1}w^*_j|\mu_j + \varepsilon|)^2}{(1 - \varepsilon)^2}\\
    &\leq \frac{(\sum_{j = 2}^{d+1}(w^*_j)^2 )(\sum_{j=2}^{d+1}(\mu_j + \varepsilon)^2)}{(1 - \varepsilon)^2}.
\end{align*}
Therefore,
\begin{align*}
    &\qquad ||w^*||^2_2 \leq \sum_{j = 2}^{d+1}(w^*_j)^2( \frac{\sum_{j=2}^{d+1}(\mu_j + \varepsilon)^2 +(1- \varepsilon)^2}{(1- \varepsilon)^2})\\
    &\Longleftrightarrow \frac{||w^*||^2_2(1- \varepsilon)^2}{\sum_{j=2}^{d+1}(\mu_j + \varepsilon)^2 +(1- \varepsilon)^2} \leq \sum_{j = 2}^{d+1}(w^*_j)^2.
\end{align*}
\end{proof}
The condition in Theorem \ref{lemma: adversarial training uses non-robust feature (general)} make sure that for any perturbation, the model would still rely on non-robust feature,
\begin{equation*}
    p < 1 - \sup_s((\frac{1}{2}(\frac{\Bar{\sigma}_{\mu,s}}{||\mu + \varepsilon s||_2} + \frac{\lambda}{2||\mu + \varepsilon s||_2^2}) + \frac{1}{2}\sqrt{\frac{2}{\lambda}}\Bar{\sigma}_{\mu,s}).
\end{equation*}
If we assume that the variance ${\sigma}_{\mu,s} \approx \sigma$ is about the same for all $s$ then the condition becomes
\begin{equation*}
    p < 1 - \sup_s((\frac{1}{2}(\frac{\sigma}{||\mu + \varepsilon s||_2} + \frac{\lambda}{2||\mu + \varepsilon s||_2^2}) + \frac{1}{2}\sqrt{\frac{2}{\lambda}}\sigma).
\end{equation*}
We know that $s^*$ that achieve the supremum would also minimize $||\mu + \varepsilon s||_2$. The optimal $s^*$ follows
\begin{enumerate}
    \item If $2\mu_i >\varepsilon $ then $s_i^* = -1$;
    \item If $\varepsilon > 2\mu_i$ then $s_i^* = 0$.
\end{enumerate}
If the perturbation budget is large enough where for all $i$, we have $\varepsilon > 2\mu_i$ then this condition is equivalent to the condition in Theorem \ref{lemma: standard training uses non-robust feature (general)}. 
\subsection{Simplified condition}
We will reduce the condition,
\begin{equation*}
    p < 1 - \sup_s((\frac{1}{2}(\frac{\Bar{\sigma}_{\mu,s}}{||\mu + \varepsilon s||_2} + \frac{\lambda}{2||\mu + \varepsilon s||_2^2}) + \frac{1}{2}\sqrt{\frac{2}{\lambda}}\Bar{\sigma}_{\mu,s}),
\end{equation*}
when
\begin{equation*}
    s \in \{-1,0,1\}^{d+1}, \Bar{\sigma}_{\mu, s} = \sqrt{ \frac{\sum_{j=1}^{d+1} (\mu_j + \varepsilon s_j)^2\sigma_j^2}{||\mu + \varepsilon s||_2^2}},
\end{equation*}
to a condition in the simplified version of Theorem \ref{lemma: adversarial training uses non-robust feature (general)} in the main text,
\begin{equation*}
    p < 1 - (\frac{1}{2}(\frac{\sigma_{\max}}{||\mu' ||_2} + \frac{\lambda}{2||\mu'||_2^2}) + \frac{1}{2}\sqrt{\frac{2}{\lambda}}\sigma_{\max}),
\end{equation*}
when
\begin{equation*}
    \sigma_i \leq \sigma_{\max}, \quad \mu' = [0, \mu_2, \dots, \mu_{d+1}].
\end{equation*}
We make assumptions that $\varepsilon > 2\mu_i$ for $i = 2,\dots, d+1$ so that for any $s \in \{-1,0,1\}^{d+1}$,
\begin{equation*}
    ||\mu + \varepsilon s||_2 > \sum_{j=2}^{d+1} \mu_j^2 = ||\mu'||^2_2
\end{equation*}
We note that the terms $\mu_1, \sigma_1^2$ are fuctions of $p$. However, we can bound them with constants
\begin{equation*}
    \mu_1 = \mathbb{E}[x_1y] = 2p-1 \leq 1,
\end{equation*}
and 
\begin{equation*}
    \sigma_1^2 = 1 - (2p-1)^2\leq 1.
\end{equation*}
We have
\begin{align*}
 \Bar{\sigma}_{\mu, s} &=\sqrt{ \frac{\sum_{j=1}^{d+1} (\mu_j + \varepsilon s_j)^2\sigma_j^2}{||\mu + \varepsilon s||_2^2}} \\
 &= \sqrt{ \frac{(2p-1 + \varepsilon s_1)^2\sigma_1^2 +\sum_{j=2}^{d+1}  (\mu_j + \varepsilon s_j)^2\sigma_{j}^2}{(2p-1 + \varepsilon s_1)^2  + \sum_{j=2}^{d+1} (\mu_j + \varepsilon s_j)^2}}\\
    &\leq \sqrt{ \frac{(2p-1 + \varepsilon s_1)^2 +\sum_{j=2}^{d+1}  (\mu_j + \varepsilon s_j)^2\sigma_{\max}^2}{(2p-1 + \varepsilon s_1)^2  + \sum_{j=2}^{d+1} (\mu_j + \varepsilon s_j)^2}}\\
    &= \sqrt{ \sigma_{\max}^2 + \frac{(1- \sigma_{\max}^2)(2p-1 + \varepsilon s_1)^2 }{(2p-1 + \varepsilon s_1)^2  + \sum_{j=2}^{d+1} (\mu_j + \varepsilon s_j)^2}}\\
    &\leq \sigma_{\max}.
\end{align*}
In the last line, we assume that $\sigma_{\max} > 1$. However, when $\sigma_{\max} \leq 1$, we can split into 2 terms
\begin{align*}
    \Bar{\sigma}_{\mu, s} &\leq \sqrt{ \sigma_{\max}^2 + \frac{(1- \sigma_{\max}^2)(1+\varepsilon)^2 }{ \sum_{j=2}^{d+1} \mu_j^2}}\\
    &\leq \sigma_{\max} + \frac{(1+\varepsilon)\sqrt{1-\sigma_{\max}^2}}{||\mu'||_2}.
\end{align*}
For simplicity, we stick with the former case, when $\sigma_{\max} > 1$.
We have
\begin{align*}
    \sup_s((\frac{1}{2}(\frac{\Bar{\sigma}_{\mu,s}}{||\mu + \varepsilon s||_2} + \frac{\lambda}{2||\mu + \varepsilon s||_2^2}) + \frac{1}{2}\sqrt{\frac{2}{\lambda}}\Bar{\sigma}_{\mu,s}) \leq  (\frac{1}{2}(\frac{\sigma_{\max}}{||\mu'||_2} + \frac{\lambda}{2||\mu'||_2^2}) + \frac{1}{2}\sqrt{\frac{2}{\lambda}}\sigma_{\max}.
\end{align*}
Therefore, if $p$ satisfies
\begin{equation*}
    p < 1 - (\frac{1}{2}(\frac{\sigma_{\max}}{||\mu' ||_2} + \frac{\lambda}{2||\mu'||_2^2}) + \frac{1}{2}\sqrt{\frac{2}{\lambda}}\sigma_{\max}),
\end{equation*}
we would have the condition in Theorem \ref{lemma: adversarial training uses non-robust feature (general)}.
\section{ADVERSARIAL TRAINING DOES NOT CONVERGE}

\begin{theorem}
\label{theorem: AT does not converge}
(AT does not converge)  Consider applying AT to learn a linear model $f(x) = w^{\top}x$ on the SVM objective when the data follows the distribution as in Definition \ref{def: distribution for AT}. Let $w^{(t)} = [w^{(t)}_1,w^{(t)}_2, \dots, w^{(t)}_{d+1}]$ be the parameter of the linear function at time $t$. If
\begin{equation}
\label{eq: con at does not converge}
    p < 1 - \sup_s((\frac{1}{2}(\frac{\Bar{\sigma}_{\mu,s}}{||\mu + \varepsilon s||_2} + \frac{\lambda}{2||\mu + \varepsilon s||_2^2}) + \frac{1}{2}\sqrt{\frac{2}{\lambda}}\Bar{\sigma}_{\mu,s}),
\end{equation}
when
\begin{equation*}
    s \in \{-1,0,1\}^{d+1}, \Bar{\sigma}_{\mu, s} = \sqrt{ \frac{\sum_{j=1}^{d+1} (\mu_j + \varepsilon s_j)^2\sigma_j^2}{||\mu + \varepsilon s||_2^2}},
\end{equation*}
then $w^{(t)}$ does not converge as $t\to \infty$. 
\end{theorem}
\begin{proof}
The difference between $w$ of two consecutive iterations is given by
\begin{align*}
    ||w^{(t+1)} - w^{(t)}||_2^2 &= \sum_{i = 1}^{d+1}(w^{(t+1)}_i - w_i^{(t)})^2.
\end{align*}
From Theorem \ref{thm: AT cycle}, for a non-robust feature $j \geq 2$, the sign of $w_j^{(t)}$, $w_j^{(t+1)}$ cannot be both positive or negative. If
\begin{equation*}
    w_j^{(t)} > 0,
\end{equation*}
then 
\begin{equation*}
    w_j^{(t+1)} \leq 0,
\end{equation*}
and if 
\begin{equation*}
    w_j^{(t)} < 0,
\end{equation*}
then 
\begin{equation*}
    w_j^{(t+1)} \geq 0.
\end{equation*}
This implies that
\begin{equation*}
    (w^{(t+1)}_j - w_j^{(t)})^2 =  (|w^{(t+1)}_j| + |w_j^{(t)}|)^2 \geq (w_j^{(t)})^2.
\end{equation*}
We have
\begin{equation*}
    ||w^{(t+1)} - w^{(t)}||_2^2 \geq \sum_{j=2}^{d+1}(w_j^{(t)})^2.
\end{equation*}
Because \eqref{eq: con at does not converge} holds, from Theorem \ref{lemma: adversarial training uses non-robust feature (general)}, we have
\begin{equation*}
    \sum_{j=2}^{d+1}(w_j^{(t)})^2 \geq \frac{||w^{(t)}||_2^2(1-\varepsilon)^2}{(1-\varepsilon)^2 + \sum_{j=2}^{d+1} (\mu_j + \varepsilon)^2}.
\end{equation*}
Therefore,
\begin{equation}
\label{eq: bound on difference w_t+1 - w_t}
    ||w^{(t+1)} - w^{(t)}||_2^2 \geq \frac{||w^{(t)}||_2^2(1-\varepsilon)^2}{(1-\varepsilon)^2 + \sum_{j=2}^{d+1} (\mu_j + \varepsilon)^2}.
\end{equation}
Assume that $w^{(t)}$ converge to $w^*$ as $t \to \infty$ then we must have
\begin{equation*}
    ||w^{(t+1)} - w^{(t)}||_2^2 \to 0,
\end{equation*}
and 
\begin{equation*}
    || w^{(t)}||_2^2 \to ||w^*||^2_2.
\end{equation*}
From inequality \eqref{eq: bound on difference w_t+1 - w_t}, take $t \to \infty$, we have
\begin{equation*}
    0 \geq \frac{||w^*||_2^2(1-\varepsilon)^2}{(1-\varepsilon)^2 + \sum_{j=2}^{d+1} (\mu_j + \varepsilon)^2}.
\end{equation*}
Therefore,
\begin{equation*}
    ||w^*||_2 = 0.
\end{equation*}
If $w^{(t)}$ converge then it can only converge to 0. However, from Lemma \ref{lem: lower bound on the magnitude} 
\begin{align*}
    ||w^{(t)}||_2 &\geq \frac{1}{||\mu+ \varepsilon s^{(t)}||_2}(1 - \frac{1}{2}(\frac{\Bar{\sigma}_{\mu,s^{(t)}}}{||\mu+ \varepsilon s^{(t)}||_2} + \frac{\lambda}{2||\mu+ \varepsilon s^{(t)}||_2^2})),\\
\end{align*}
when
\begin{equation*}
    \Bar{\sigma}_{\mu, s^{(t)}} = \sqrt{ \frac{\sum_{j=1}^{d+1} (\mu_j + \varepsilon s_j^{(t)})^2\sigma_j^2}{||\mu + \varepsilon s^{(t)}||_2^2}},
\end{equation*}
and $s^{(t)} = \frac{1}{y\varepsilon}\delta^{(t)}$ is the sign of the perturbation at time $t$. $||w^{(t)}||_2$ is bounded below therefore it cannot converge to zero. This leads to a contradiction. We can conclude that $w^{(t)}$ does not converge as $t\to \infty$.

\end{proof}


\section{NASH EQUILIBRIUM IS ROBUST}\label{appendix: Nash equilibrium is robust}
\begin{proof}
Let $(\delta^*,w^*)$ be a Nash equilibrium.
Let $x_i$ be a non-robust feature. We will show that $w_i^* = 0$ by contradiction. Without loss of generality, let $w_i^* > 0$. Let the risk term in the SVM objective when $w_i = w$, $w_j = w_j^*$ for $j \neq i$ and $\delta = \delta^*$ be
\begin{align*}
    \mathcal{L}_i(w|w^*, \delta^*) &:= \mathbb{E}[l_i(x, y, w|w^*, \delta^*)].
\end{align*}
when 
\begin{align*}
    l_i(x, y, w|w^*, \delta^*) &= \max(0, 1 - y\sum_{j \neq i}w_j^*(x_j + \delta_j^*(x,y))) - yw(x_i + \delta_i^*(x,y)).
\end{align*}
We will show that when we set $w_i^* = 0$, the risk term does not increase, that is,
\begin{equation*}
    \mathcal{L}_i(w_i^*|w^*, \delta^*) \geq \mathcal{L}_i(0|w^*, \delta^*).
\end{equation*}
We use $l_i(x,y,w)$ to refer to $ l_i(x,y, w|w^*, \delta^*)$ for the rest of this proof. Considering each point $(x,y)$, we have $2$ cases:\\

\noindent\textbf{Case 1: }
\begin{equation*}
    1 - y\sum_{j} w_j^* x_j + \varepsilon \sum_{j} |w_j^*| > 0.
\end{equation*}\\
From Lemma \ref{lem: optimal perturbation short}, we have
\begin{equation*}
    \delta_j^*(x,y) = -y\varepsilon \operatorname{sign}(w_j),
\end{equation*}
for all $j$ with $w_j^* \neq 0$.\\

\noindent\textbf{Case 1.1: }
\begin{equation*}
    1 - y\sum_{j\neq i} w_j^* x_j + \varepsilon \sum_{j \neq i} |w_j^*| \geq 0.
\end{equation*}\\
In this case, we have
\begin{align*}
    l_i(x,y,w_i^*) - l_i(x,y,0) 
    &= \max(0, 1 - y\sum_{j }w_j^*(x_j + \delta_j^*(x,y)) ) - \max(0, 1 - y\sum_{j \neq i}w_j^*(x_j + \delta_j^*(x,y)) ) \\
    &= \max(0, 1 - y\sum_{j} w_j^* x_j + \varepsilon \sum_{j} |w_j^*| ) - \max(0, 1 - y\sum_{j \neq i} w_j^* x_j + \varepsilon \sum_{j \neq i} |w_j^*| ) \\
    &= (1 - y\sum_{j} w_j^* x_j + \varepsilon \sum_{j} |w_j^*|) - (1 - y\sum_{j \neq i} w_j^* x_j + \varepsilon \sum_{j \neq i} |w_j^*|) \\
    &= -yw_i^*x_i + \varepsilon |w_i^*|.
\end{align*} \\

\noindent\textbf{Case 1.2: }
\begin{equation*}
    1 - y\sum_{j\neq i} w_j^* x_j + \varepsilon \sum_{j \neq i} |w_j^*| < 0.
\end{equation*}\\
In this case, we have
\begin{align*}
    l_i(x, y,w_i^*) - l_i(x,y, 0) \\&= \max(0, 1 - y\sum_{j }w_j^*(x_j + \delta_j^*(x,y)) ) - \max(0, 1 - y\sum_{j \neq i}w_j^*(x_j + \delta_j^*(x,y)) ) \\
    &= \max(0, 1 - y\sum_{j} w_j^* x_j + \varepsilon \sum_{j} |w_j^*| ) - \max(0, 1 - y\sum_{j \neq i} w_j^* x_j + \varepsilon \sum_{j \neq i} |w_j^*| ) \\
    &= (1 - y\sum_{j} w_j^* x_j + \varepsilon \sum_{j} |w_j^*|) - 0 \\
    &\geq 0.
\end{align*}\\

\noindent\textbf{Case 2: }
\begin{equation*}
    1 - y\sum_{j} w_j^* x_j + \varepsilon \sum_{j} |w_j^*| \leq 0.
\end{equation*}\\
From Lemma \ref{lem: optimal perturbation short}, $\delta_j^*(x,y)$ can take any value in $[-\varepsilon, \varepsilon]$ and 
\begin{equation*}
    \max(0, 1 - y\sum_{j }w_j^*(x_j + \delta_j^*(x,y)) ) = 0.
\end{equation*}\\

\noindent\textbf{Case 2.1:}
\begin{equation*}
    1 - y\sum_{j\neq i} w_j^* x_j + \varepsilon \sum_{j \neq i} |w_j^*| \geq 0.
\end{equation*}
This implies that
\begin{equation*}
    -yw_i^*x_i + \varepsilon|w_i^*| \leq 0.
\end{equation*}
We have 2 further cases:\\

\noindent\textbf{Case 2.1.1: }
\begin{equation*}
    1 - y\sum_{j\neq i} w_j^* (x_j + \delta_j^*(x,y))\geq 0.
\end{equation*}\\
In this case, we have
\begin{align*}
    l_i(x,y, w_i^*) - l_i(x,y, 0) 
    &= \max(0, 1 - y\sum_{j }w_j^*(x_j + \delta_j^*(x,y)) ) - \max(0, 1 - y\sum_{j \neq i}w_j^*(x_j + \delta_j^*(x,y)) ) \\
    &= 0 - (1 - y\sum_{j \neq i}w_j^*(x_j + \delta_j^*(x,y))) \\
    &= - (1 - y\sum_{j \neq i}w_j^*(x_j + \delta_j^*(x,y))) \\
    &\geq - (1 - y\sum_{j \neq i}w_j^*x_j + \varepsilon \sum_{j \neq i} |w_j^*|)\\
    &\geq -yw_i^*x_i + \varepsilon |w_i^*|.
\end{align*} 
The final inequality holds since
\begin{equation*}
    1 - y\sum_{j} w_j^* x_j + \varepsilon \sum_{j} |w_j^*| \leq 0,
\end{equation*}
which implies
\begin{align*}
    -yw_i^*x_i + \varepsilon |w_i^*| &\leq - (1 - y\sum_{j \neq i}w_j^*x_j + \varepsilon \sum_{j \neq i} |w_j^*|).
\end{align*}\\

\noindent\textbf{Case 2.1.2: }
\begin{equation*}
    1 - y\sum_{j\neq i} w_j^* (x_j + \delta_j^*(x,y)) < 0.
\end{equation*}\\
In this case, we have
\begin{align*}
    l_i(x, w_i^*) - l_i(x, 0)&= \max(0, 1 - y\sum_{j }w_j^*(x_j + \delta_j^*(x,y)) ) - \max(0, 1 - y\sum_{j \neq i}w_j^*(x_j + \delta_j^*(x,y)) ) \\
    &= 0 - 0 \\
    &\geq -yw_i^*x_i + \varepsilon |w_i^*|.
\end{align*} \\
The last inequality holds because we know that
\begin{equation*}
    -yw_i^*x_i + \varepsilon|w_i^*| \leq 0.
\end{equation*}\\

\noindent \textbf{Case 2.2: }
\begin{equation*}
    1 - y\sum_{j\neq i} w_j^* x_j + \varepsilon \sum_{j \neq i} |w_j^*| < 0.
\end{equation*}
In this case, we have
\begin{equation*}
    1 - y\sum_{j\neq i} w_j^* (x_j + \delta_j^*(x,y)) < 1 - y\sum_{j\neq i} w_j^* x_j + \varepsilon \sum_{j \neq i} |w_j^*| < 0,
\end{equation*}
and 
\begin{align*}
    l_i(x, w_i^*) - l_i(x, 0) &= \max(0, 1 - y\sum_{j }w_j^*(x_j + \delta_j^*(x,y)) ) - \max(0, 1 - y\sum_{j \neq i}w_j^*(x_j + \delta_j^*(x,y)) ) \\
    &= 0 - 0 = 0. \\
\end{align*}

\noindent From every case, we can conclude that 
\begin{align*}
    \mathcal{L}_i(w_i^*|w_*, \delta^*) - \mathcal{L}_i(0|w*, \delta^*) &:= \mathbb{E}[l_i(x,y, w_i^*) - l_i(x,y,0)]\\
    &\geq \mathbb{E}[(-yw_i^*x_i + \varepsilon |w_i^*|) \mathds{1} \{ 1 - y\sum_{j\neq i} w_j^* x_j + \varepsilon \sum_{j \neq i} |w_j^*| \geq 0 \}] \\
    &= \mathbb{E}[-yw_i^*x_i + \varepsilon |w_i^*|]\mathbb{P}\left(1 - y\sum_{j\neq i} w_j^* x_j + \varepsilon \sum_{j \neq i} |w_j^*| \geq 0\right) \\
    &\geq |w_i^*|(\varepsilon - |\mu_i|)\mathbb{P}\left(1 - y\sum_{j\neq i} w_j^* x_j + \varepsilon \sum_{j \neq i} |w_j^*| \geq 0\right)\\
    &\geq 0,
\end{align*}
where the third line holds since the features are independent of each other. The risk term in the utility when $w_i \neq 0$ is no better than when $w_i = 0$. However, the regularization term is higher, 
\begin{equation*}
    \frac{\lambda}{2}\sum_{j}(w_j^*)^2 > \frac{\lambda}{2}\sum_{j \neq i} (w_j^*)^2.
\end{equation*}
Therefore, we can reduce the SVM objective by setting $w_i^* = 0$. This contradicts with the optimality of $w^*$. By contradiction, we can conclude that if a feature $i$ is not robust, then $w_i^* = 0$.
\end{proof}
\section{OPTIMAL ADVERSARIAL TRAINING LEADS TO A ROBUST MODEL}\label{appendix: OAT is robust}
\begin{proof}
We are learning a function $f(x) = w^{\top}x$ where $w = [w_1, \dots , w_d] \in \mathbb{R}^d$. For a fixed $w$, we know that the perturbation that maximizes the inner loss is $\delta^*(x,y) = -y\varepsilon \operatorname{sign}(w)$. Substitute this in the objective, we are left to solve
\begin{equation}
    \label{eq outer objective}
    \min_w \mathbb{E}_{(x,y)\sim D}[\max(0, 1 - yw^{\top}x + \varepsilon||w||_1)] + \frac{\lambda}{2}||w||^2_2.
\end{equation}
Assume that $w^*$ is an optimal solution of \eqref{eq outer objective}. For a non-robust feature $x_i$, we will show that $w_i^* = 0$ by contradiction. Assume that $|w_i^*| > 0$. Consider the expected contribution of $w_i^*$ to the first term of the objective (risk) is given by
\begin{align*}
    &\mathbb{E}_{(x,y)\sim D}[\max(0, 1 - yw^{\top}x + \varepsilon||w||_1)]\\
    &= \mathbb{E}_{(x,y)\sim D}[\max(0, 1 + \sum_{j}(\varepsilon |w_j| - yw_jx_j ) )]\\
    &= \mathbb{E}_{(x,y)\sim D}[\max(0, 1 + \sum_{j} p_j )]\\
    &= \mathbb{E}_{(x,y)\sim D}[\max(-1 - \sum_{j \neq i}p_j, p_i )+ 1 + \sum_{j \neq i}p_j],
\end{align*}
when we denote $p_j = \varepsilon |w_j| - yw_jx_j$. Since $\max(0, \cdot)$ is a convex function, by Jensen's inequality, we have
\begin{align*}
    &\mathbb{E}_{(x,y)\sim D}[\max(-1 - \sum_{j \neq i}p_j, p_i )+ 1 + \sum_{j \neq i}p_j] \\
    &= \mathbb{E}_{y}\mathbb{E}_{x_j|y}\mathbb{E}_{x_i|y}[\max(-1 - \sum_{j \neq i}p_j, p_i )+ 1 + \sum_{j \neq i}p_j] \\
    &\geq \mathbb{E}_{y}\mathbb{E}_{x_j|y}[\max(-1 - \sum_{j \neq i}p_j, \mathbb{E}_{x_i|y}[p_i] )+ 1 + \sum_{j \neq i}p_j].
\end{align*}
Because features are independent, we can split the expectation between $x_i, x_j$. We note that as feature $x_i$ is non-robust, we have $|\mu_i| \leq \varepsilon$ so that
\begin{align*}
    \mathbb{E}_{x_i|y}[p_i] &= \mathbb{E}_{x_i|y}[\varepsilon |w_i| - yw_ix_i]\\
    &= \varepsilon |w_i| - \mathbb{E}_{x_i|y}[yw_ix_i]\\
    &\geq |w_i|(\varepsilon - |\mu_i|)\\
    &\geq 0.
\end{align*}
This implies that
\begin{align*}
    &\mathbb{E}_{y}\mathbb{E}_{x_j|y}[\max(-1 - \sum_{j \neq i}p_j, \mathbb{E}_{x_i|y}[p_i] )+ 1 + \sum_{j \neq i}p_j] \\
    &\geq \mathbb{E}_{y}\mathbb{E}_{x_j|y}[\max(-1 - \sum_{j \neq i}p_j,0 )+ 1 + \sum_{j \neq i}p_j]\\
    &= \mathbb{E}_{y}\mathbb{E}_{x_j|y}[\max(0, 1 + \sum_{j \neq i}p_j)].
\end{align*}
The right-hand side term is just the loss term when we set $w_i = 0$. Therefore, setting $w_i = 0$ for non-robust features does not increase the loss. At the same time, setting $w_i = 0$ reduces the second term of the objective $\frac{\lambda}{2}||w||^2_2$. Thus, we can reduce the objective \eqref{eq outer objective} by setting $w_i = 0$ for non-robust feature $i$ . This contradicts the optimality of $w^*$. By contradiction, we have $w_i^* = 0$ for all feature $x_i$ that is non-robust.
\end{proof}

\section{EXPERIMENT}\label{appendix: synthetic experiment}
The loss function is a hinge loss with an $\ell_2$ regularization of $\lambda = 0.01$. We train a linear model with AT and OAT for 50 time steps. For OAT, we directly optimize the weight with respect to the loss
\begin{equation*}
    \mathcal{L}_{OAT}(w)    \sum_{i=1}^{200}[\max(0, 1 - y_iw^{\top}x_i + \varepsilon||w||_1)] + \frac{\lambda}{2}||w||^2_2.
\end{equation*}
We use an Adam optimizer \citep{kingma2014adam} with a learning rate $0.01$ and batch size $200$.\\

For AT, at each time step $t$, we first generate the worst-case perturbations 
\begin{equation*}
        \delta^{(t)}(x,y) = -y\varepsilon\operatorname{sign}(w^{(t-1)})
\end{equation*}
then we generate adversarial examples accordingly. Next, we update our model with an Adam optimizer with a learning rate $0.01$ and a batch size of $200$. The loss of each batch is given by
\begin{equation*}
        \mathcal{L}(w)_{AT} = \sum_{i=1}^{200} \max(0, 1 - y_iw^{\top}(x_i+ \delta^{(t)}(x_i,y_i))) + \frac{\lambda}{2}||w||^2_2.\\
\end{equation*}



\end{document}